\documentclass[10pt,twocolumn,letterpaper]{article}

\usepackage[pagenumbers]{cvpr} % To force page numbers, e.g. for an arXiv version

\usepackage{epsfig}
\usepackage{graphicx}
\usepackage{amsmath}
\usepackage{amssymb}
\usepackage{booktabs}
\usepackage{multirow}
\usepackage{lipsum}
\usepackage{xhfill}
\usepackage{mathtools}
\usepackage[accsupp]{axessibility}  % Improves PDF readability for those with disabilities.

\usepackage[pagebackref,breaklinks,colorlinks]{hyperref}

\usepackage[capitalize]{cleveref}
\crefname{section}{Sec.}{Secs.}
\Crefname{section}{Section}{Sections}
\Crefname{table}{Table}{Tables}
\crefname{table}{Tab.}{Tabs.}

\newcommand{\bacon}{\textsc{Bacon}}
\newcommand{\siren}{\textsc{Siren}}
\newcommand{\nglod}{\textsc{Nglod}}
\newcommand*{\red}{\textcolor{black}}

\begin{document}
%%%%%%%%% TITLE
\title{\textsc{Bacon}: Band-limited Coordinate Networks\\ for Multiscale Scene Representation\vspace{-0.5em}}

\author{David B. Lindell
\qquad Dave Van Veen
\qquad Jeong Joon Park
\qquad Gordon Wetzstein\\[0.25em]
Stanford University\\
{\small\url{http://computationalimaging.org/publications/bacon}}
}

\maketitle

%%%%%%%%% ABSTRACT
\begin{abstract}

% set context, why now
Coordinate-based networks have emerged as a powerful tool for 3D representation and scene reconstruction.
These networks are trained to map continuous input coordinates to the value of a signal at each point.
% what is the problem with current methods
Still, current architectures are black boxes: their spectral characteristics cannot be easily analyzed, and their behavior at unsupervised points is difficult to predict.
Moreover, these networks are typically trained to represent a signal at a single scale, so naive downsampling or upsampling results in artifacts.
% what is our approach 
We introduce band-limited coordinate networks (\bacon{}), a network architecture with an analytical Fourier spectrum.
\red{\bacon{} has constrained behavior at unsupervised points, can be designed based on the spectral characteristics of the represented signal, and can represent signals at multiple scales without per-scale supervision.}
We demonstrate \bacon{} for multiscale neural representation of images, radiance fields, and 3D scenes using signed distance functions and show that it outperforms conventional single-scale coordinate networks in terms of interpretability and quality.
\end{abstract}

%%%%%%%%% BODY TEXT

\vspace{-1em}
\section{Introduction}
\label{sec:introduction}
\begin{figure}
    \centering
    \includegraphics[width=\columnwidth]{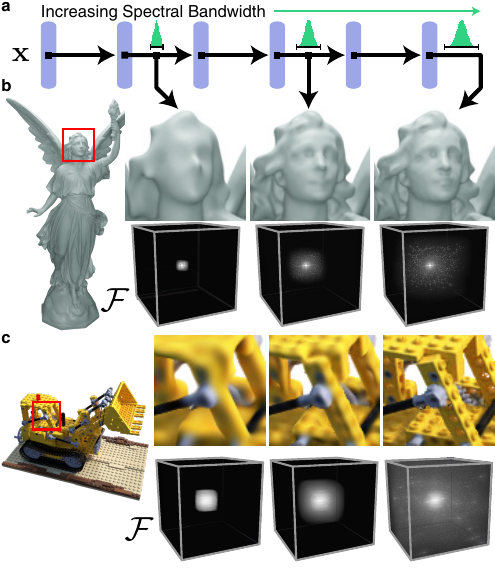}
    \caption{Overview of band-limited coordinate networks (\bacon{}). (a) The proposed architecture produces intermediate outputs with an analytical spectral bandwidth that can be specified at initialization. When supervised on a high-resolution signal, the network learns a multi-resolution decomposition of the output, as shown for fitting 3D shapes via a signed distance function (b) and radiance fields (c). The network is characterized entirely by its Fourier spectrum (see insets) so its behavior is \red{constrained}, even at unsupervised locations.}
    \label{fig:teaser}
    \vspace{-1em}
\end{figure}

Coordinate networks are an emerging class of neural networks that can be used to represent or optimize a broad format of signals including images, video, 3D models, audio waveforms, and more~\cite{park2019deepsdf,mescheder2019occupancy,sitzmann2019srns,mildenhall2020nerf,sitzmann2020siren}.
As opposed to storing discrete samples of signals in conventional array- or grid-based formats, neural representations approximate signals using a continuous function that is embedded in the learned weights of a fully-connected neural network.
Given an input coordinate, these networks are trained to output the value of a signal at that point.
Since even complex or high-dimensional signals can be flexibly optimized using a coordinate network, they have become popular for applications including view synthesis~\cite{mildenhall2020nerf}, image processing~\cite{saito2019pifu}, 3D reconstruction~\cite{park2019deepsdf}, and neural rendering~\cite{tewari2020state}.

Yet, current coordinate networks are black box models that are designed to represent signals at a single scale.
As a result, the behavior of the network at unsupervised coordinates is difficult to predict, with complex dependencies on hyperparameters such as hidden layer size, network depth, or input coordinate encoding.
The black box nature of the architecture similarly inhibits multiscale signal representation, since we cannot readily filter or anti-alias these models, and the frequency spectrum of a coordinate network is difficult to analyze.
Thus naive downsampling or upsampling by querying the network on a coarser or finer grid of coordinates leads to aliasing or undesired high-frequency artifacts.
Ultimately, these characteristics stem from the fact that coordinate networks are not amenable to Fourier analysis and are not designed to be scale aware.

Still, being able to represent and optimize signals at multiple resolutions is an important requirement for many applications.
For example in image processing, many techniques rely on image pyramids~\cite{simoncelli1995steerable} (e.g., optical flow estimation, compression, filtering, etc.). 
Representing 3D objects or scenes at multiple levels of detail is useful for speeding up rendering and reducing memory requirements (e.g., mipmapping).

In this work, we introduce band-limited coordinate networks (\bacon{}). 
The key properties of this architecture are that (1) the maximum frequency at each layer can be manipulated analytically, and (2) the behavior of a trained network is entirely characterized by its Fourier spectrum. 
\bacon{} is suited to multiscale signal representation because band-limited output layers can be designed with an inductive bias towards a particular resolution or scale.

In addition to introducing \bacon{}, we demonstrate a variety of applications including multiscale representation of images, neural radiance fields, and 3D scenes.
Our work takes important steps towards making coordinate-based networks scale aware, and provides a new representation with interpretable behavior. 
Specifically, we make the following contributions:
\begin{itemize}
    \item We introduce band-limited coordinate-based networks for representing and optimizing signals.
    \item We develop methods for spectral analysis of the architecture, and propose a principled, band-limited initialization scheme.
    \item We demonstrate that our architecture outperforms conventional single-scale coordinate networks for multiscale image fitting, neural rendering, and 3D scene representation.
\end{itemize}

\section{Related Work}
\label{sec:related}
\begin{figure*}[ht!]
    \centering
    \includegraphics{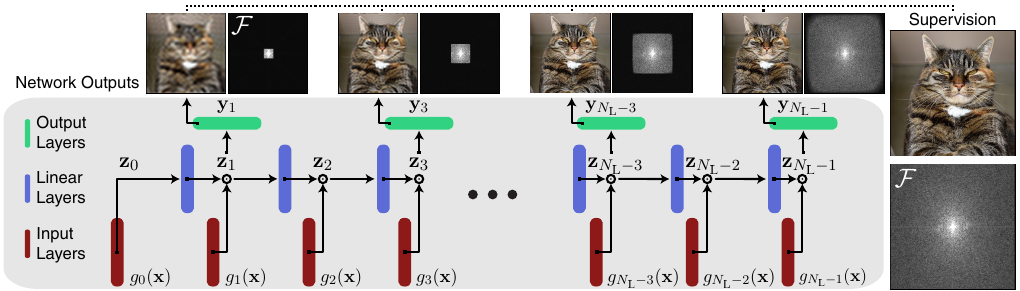}
    \caption{Overview of \bacon{} architecture. We initialize the frequencies of the sine layers of a multiplicative filter network~\cite{fathony2020multiplicative} within a limited bandwidth $[-B_i, B_i]$ (bottom row). Then, the bandwidth of each output layer is the sum of the input bandwidths up to that point (top row), allowing the network bandwidth to be explicitly specified. At training time the network can be supervised with a signal at any resolution, and the network learns to fit the signal in a band-limited fashion. Image from DIV2K dataset~\cite{Agustsson_2017_CVPR_Workshops}.}
    \label{fig:architecture}
    \vspace{-1em}
\end{figure*}

\paragraph{Neural Scene Representation and Rendering.}

Emerging neural scene representations promise 3D-structure-aware, continuous, memory-efficient representations for parts~\cite{genova2019learning,genova2019deep}, objects~\cite{park2019deepsdf,michalkiewicz2019implicit,atzmon2019sal,gropp2020implicit,yariv2020multiview,davies2020effectiveness,chabra2020deep}, or scenes~\cite{eslami2018neural,sitzmann2019srns,jiang2020local,peng2020convolutional,sitzmann2020siren}.
These can be supervised with 3D data, such as point clouds, and optimized as either signed distance functions~\cite{park2019deepsdf,michalkiewicz2019implicit,atzmon2019sal,gropp2020implicit,sitzmann2019srns,jiang2020local,peng2020convolutional,sitzmann2019metasdf,kellnhofer2021neural,yifan2021geometry,takikawa2021nglod} or occupancy networks~\cite{mescheder2019occupancy,chen2019learning}.
Using neural rendering~\cite{tewari2020state, tewari2021advances}, representation networks can also be trained using multiview 2D images~\cite{barron2021mip,saito2019pifu,sitzmann2019srns,Oechsle2019ICCV,Niemeyer2020CVPR,mildenhall2020nerf,yariv2020multiview,liu2020neural,jiang2020sdfdiff,liu2020dist,martinbrualla2020nerfw,pumarola2020d,srinivasan2020nerv,zhang2020nerf,neff2021donerf,oechsle2021unisurf,garbin2021fastnerf,lindell2020autoint,wang2021neus,yariv2021volume,yu2020pixelnerf}.
Temporally aware extensions ~\cite{Niemeyer2019ICCV} and multimodal variants with part-level semantic segmentation~\cite{kohli2020inferring} have also been proposed.
\red{Recent 2D GANs have analyzed the bandwidth of convolutional layers for image generation~\cite{karras2021alias}, and 
3D-aware GANs use related ideas but are trained with 2D image collections~\cite{graf,chan2020pi,Niemeyer2020GIRAFFE,chan2022efficient,or2021stylesdf,deng2021gram}.}

\vspace{-1.5em}
\paragraph{Architectures for Scene Representation.}

Neural network architectures for scene representation networks can be roughly classified as feature-based, coordinate-based, or hybrid.
Feature-based approaches represent the scene using differentiable feature primitives, such as points~\cite{fan2017point,qi2017pointnet,wang2019dynamic,yifan2021iso,Peng2021NEURIPS}, surface patches~\cite{yifan2019differentiable}, meshes~\cite{hedman2018deep,thies2019deferred,riegler2020free,zhang2020neural}, multi-plane~\cite{Zhou:2018,Mildenhall:2019,flynn2019deepview} or multi-sphere~\cite{Broxton:2020,Attal:2020:ECCV} images, or using a voxel grid of features~\cite{sitzmann2019deepvoxels,Lombardi:2019}.
A tradeoff with feature-based representations is that they can be quickly evaluated, but typically have a large memory footprint.

Coordinate-based representations (sometimes called implicit representations or coordinate networks), use a multilayer perceptron (MLP) to map input coordinates to a signal value, for example, the signed distance or occupancy of a 3D scene.
% Ultimately, the maximum resolution of coordinate networks is tied to the capacity of the network architecture through dependencies on the number of trainable parameters and other architectural choices.
These networks can represent signals globally~\cite{park2019deepsdf,mescheder2019occupancy,tancik2020fourier,sitzmann2020siren} or locally~\cite{chabra2020deep,jiang2020local,chen2021learning,Reiser2021ICCV,mehta2021modulated}.
\red{Some global networks, such as Fourier Features~\cite{tancik2020fourier} and \textsc{SIREN}~\cite{sitzmann2020siren}, have tunable parameters that bias the network to fitting low- or high-frequency signals~\cite{yifan2021geometry}, though without explicit control over the bandwidth.}

Hybrid architectures combine feature-based and coordinate representations to achieve best of both worlds~\cite{liu2020neural,martel2021acorn,hedman2021snerg,peng2020convolutional}.
These networks can represent complex, high-dimensional signals continuously across the input domain with a small memory footprint.
The proposed method is also a coordinate network, but rather than using an MLP architecture, as with all coordinate networks discussed above, our method builds on recently proposed multiplicative filter networks (MFNs)~\cite{fathony2020multiplicative}.
We develop the theory of MFNs, with new tools to describe and manipulate the Fourier spectra of these networks, and a new initialization scheme that mitigates vanishing activations in deep networks.
These insights enable band-limited coordinate networks, which we demonstrate for multiscale signal representation.

\vspace{-1.5em}
\paragraph{Multiscale Representations.}

Several existing works have explored multiscale architectures in the context of scene representation networks.
For example, proposed methods use an octree~\cite{takikawa2021nglod,yu2021plenoctrees} to accelerate neural rendering of radiance fields or signed distance functions, or a hierarchy of features~\cite{chen2021multiresolution} to improve 3D shape completion.
Multiscale representations can be optimized directly using specialized architectures~\cite{martel2021acorn,yifan2021geometry} or progressive training strategies~\cite{martel2021acorn,hertz2021sape}.
The closest work to ours in this category is Mip-NeRF~\cite{barron2021mip}, which is a coordinate-based network with a scale-dependent positional encoding. 
After training the network with supervision at multiple scales, the resolution of the network output can be controlled by adjusting the positional encoding.
Our work differs in that the bandwidth of the network outputs are constrained by design rather than through training.
Thus, our approach learns a band-limited multiscale decomposition of a signal, even without explicit training at multiple scales.

\section{Method}
\label{sec:method}
This section provides an overview of MFNs and the \bacon{} multiscale architecture, describes the Fourier spectra of these networks, and proposes an initialization scheme for deep networks.

\subsection{Band-limited Coordinate Networks}

Our approach builds on a recently introduced coordinate-based architecture called Multiplicative Filter Networks (MFNs)~\cite{fathony2020multiplicative}, which differ from conventional MLPs in that they employ a Hadamard product between linear layers and sine activation functions.
While \bacon{} uses an MFN backbone, we significantly extend the theoretical understanding and practicality of these networks by (1) proposing architectural changes to achieve multiscale, band-limited outputs, (2) deriving formulas to quantify the expected frequencies in the representation, and (3) deriving a principled initialization scheme that prevents vanishing activations in deep networks.

In a forward pass through the network, an input coordinate $\mathbf{x}\in\mathbb{R}^{d_\text{in}}$ is first passed through several layers of the form $g_i : \mathbb{R}^{d_\text{in}} \mapsto \mathbb{R}^{d_\text{h}}$, with $g_i(\mathbf{x}) = \sin(\boldsymbol{\omega}_i \mathbf{x} + \boldsymbol{\phi}_i)$, $i = 0, \ldots, N_\text{L}-1$, and $N_\text{L}$ the number of layers in the network.
We refer to the intermediate activations as $\mathbf{z}_i\in\mathbb{R}^{d_\text{h}}$, and we allow intermediate outputs of the network $\mathbf{y}_i\in\mathbb{R}^{d_\text{out}}$ at the $i$th layer, defined as follows (see also Fig.~\ref{fig:architecture}).
\begin{equation}
    \begin{aligned}
    &\mathbf{z}_0 = g_0(\mathbf{x})\\
    &\mathbf{z}_{i} = g_{i}(\mathbf{x}) \circ \left(\mathbf{W}_{i} \mathbf{z}_{i-1} + \mathbf{b}_{i}\right),\quad 0\leq i<N_\text{L}\\
    &\mathbf{y}_i = \mathbf{W}^\text{out}_{i} \mathbf{z}_{i} + \mathbf{b}^\text{out}_{i},\\
    \end{aligned}
   \label{eq:mfn}
\end{equation}
where $\circ$ indicates the Hadamard product.
The parameters of the network are $\theta = \{\boldsymbol{\omega}_i\in\mathbb{R}^{{d_\text{h}} \times d_\text{in}}, \mathbf{b}_i, \boldsymbol{\phi}_i\in\mathbb{R}^{d_\text{h}}, \mathbf{W}_i\in\mathbb{R}^{d_\text{h}\times d_\text{h}}, \mathbf{W}^\text{out}_i \in \mathbb{R}^{d_\text{out}\times d_\text{h}}, b_i^\text{out} \in \mathbb{R}^{d_\text{out}}\}$.

A useful property of this formulation is that the network output can be expressed equivalently as a sum of sines with varying amplitude, frequency, and phase~\cite{fathony2020multiplicative}.
\vspace{-0.5em}
\begin{equation}
    \mathbf{y}_i = \sum\limits_{j=0}^{N_\text{sine}^{(i)}-1} \bar{\alpha}_j\sin(\boldsymbol{\bar{\omega}}_j\mathbf{x} + \bar{\phi}_j),
\end{equation}
where $\bar{\alpha}_i$, $\boldsymbol{\bar{\omega}}_i$, and $\bar{\phi}_i$ depend on the parameters of the MFN \red{(see supplemental \S1.2)}, and the number of terms in the sum for an $N_\text{L}$ layer network is given as \red{(see supplemental \S1.1)}
\begin{equation}
N_\text{sine}^{(N_\text{L})} = \sum_{i=0}^{N_\text{L}-1} 2^{i}d_\text{h}^{i+1}.
\label{eq:num_sines}
\end{equation}
This property stems from the repeated Hadamard product of sines and the trigonometric identity that 
\begin{equation}
    \sin(a) \sin(b) = \frac{1}{2} \left(\sin(a + b - \pi/2) + \sin(a - b + \pi/2)\right).
    \label{eq:sines}
\end{equation}
By applying this identity through the layers of the network, the output can be reduced to a single sum of sines.

\subsection{Frequency Spectrum}
We exploit the property that MFNs can be expressed as a sum of sines to create band-limited networks.
This is achieved by designing the architecture so that the frequency of all represented sines never exceeds a desired threshold.

To this end, we freeze (i.e., do not optimize) the frequencies, or entries of $\boldsymbol{\omega}_i$, and set them to a bandwidth in $[-B_i, B_i]$ using random uniform initialization.  
Then, since \red{the Hadamard products} of sines result in summed frequencies (Eq.~\ref{eq:sines}), the total bandwidth of an output at layer $i$ of the network is less than or equal to $\sum_{j=0}^{i} B_j$ and the maximum bandwidth is $B = \sum_{i=0}^{N_\text{L}-1}B_i$ (see Fig.~\ref{fig:architecture}).

When representing signals across a finite input domain, e.g., with input coordinates $\mathbf{x}\in[-0.5, 0.5]^{d_\text{in}}$, it is not necessary to represent all frequencies continuously. 
Instead, we can assume that the represented signal is periodic, so we are only required to represent discrete frequency values whose spacing is $1/T$, where $T$ is the periodicity or extent of the signal in the primal domain. 
Moreover, using discrete frequencies allows complete characterization of the network spectrum by applying a fast Fourier transform to a uniformly sampled network output (shown in Fig.~\ref{fig:architecture} for image fitting).

We also analyze the distribution of sine frequencies in the network.
Briefly, sines in the network can be associated with one of the $N_\text{L}$ terms in the summation of Eq.~\ref{eq:num_sines}.
Then, considering the probability of sines originating from each term results in a compound random variable that gives the overall distribution of frequencies. 
We provide an extended derivation in the supplemental, showing that the distribution is approximately zero-mean Gaussian with variance
\begin{equation}
    \text{Var}(\boldsymbol{\omega}_i)\cdot\sum\limits_{m=0}^{N_\text{L}-1} m \cdot \frac{2^{N_\text{L}-1-m}d_\text{h}^{N_\text{L}-m}}{\sum\limits_{i = 0}^{N_\text{L}-1} 2^{i}d_\text{h}^{i+1}}.
\end{equation}
The Gaussian distribution of frequencies results in a greater parameterization of low frequencies in the network; this may be a useful inductive bias since low-frequency Fourier coefficients typically have a greater amplitude than high-frequency coefficients in natural signals~\cite{torralba2003statistics}.

\begin{figure}[ht!]
    \centering
    \includegraphics{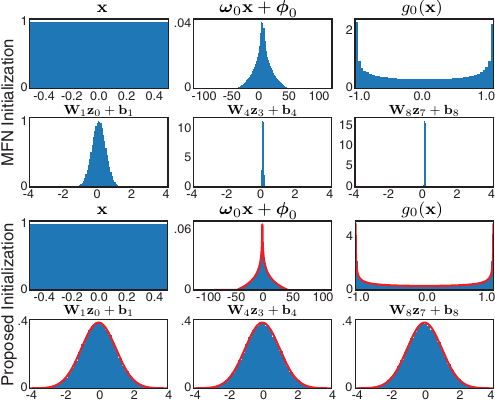}
    \caption{Comparison of distribution of activations at initialization. The initialization scheme proposed for MFNs~\cite{fathony2020multiplicative} results in vanishingly small activations for deep networks (shown for layers 0, 1, 4, and 8). The proposed initialization scheme maintains a standard normal distribution after each linear layer (all distributions shown for a network with $d_\text{h}=1024$), and activations at intermediate outputs closely match our analytical derivations (red lines, see supplemental for details).}
   \label{fig:initialization}
   \vspace{-1.5em}
\end{figure}

To facilitate representing signals at multiple resolutions, we introduce linear layers at intermediate stages throughout the network to extract band-limited outputs (see Fig.~\ref{fig:architecture}).
By supervising the outputs of these layers, we can train \bacon{} to fit a signal at multiple scales simultaneously. 
Interestingly, because the outputs are band-limited, \bacon{} can be trained in a semi-supervised fashion where the bandwidth of the supervisory signal need not match the desired bandwidth of the output of the network, demonstrated in Fig.~\ref{fig:architecture} for image fitting. 

\subsection{Initialization Scheme}

\begin{figure*}[ht]
    \centering
    \includegraphics[width=\textwidth]{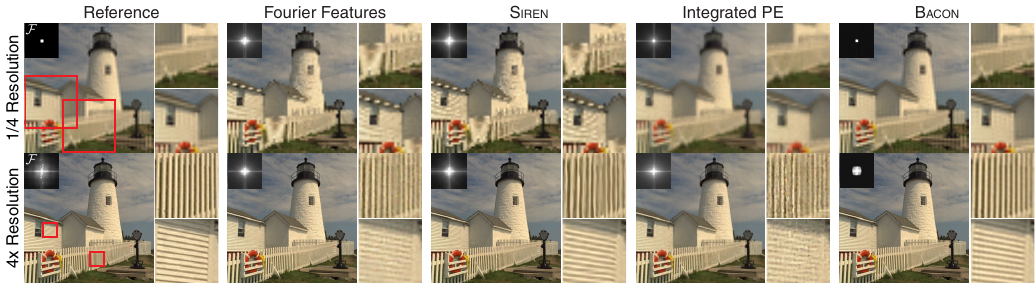}
    \caption{Image fitting results. We train networks using Fourier Features~\cite{tancik2020fourier}, \siren{}~\cite{sitzmann2020siren}, and integrated positional encoding (PE)~\cite{barron2021mip} to fit an image at 256$\times$256 (1$\times$) resolution.
    We show network outputs at 1/4 and 4$\times$ resolution. 
Fourier Features and \textsc{Siren} fit to a single scale and show aliasing when subsampled. 
Integrated PE is explicitly supervised at 1/4 and 1$\times$ resolution and learns reasonable anti-aliasing; however, all methods except \bacon{} show high-frequency artifacts at 4$\times$ resolution (insets).
\red{\bacon{} is supervised at a single scale and approximates low-pass filtered and high-resolution reference images (left column and Fourier spectra insets).}}
    \label{fig:image}
    \vspace{-1.5em}
\end{figure*}

\begin{figure}[ht]
    \centering
    \includegraphics[width=0.95\columnwidth]{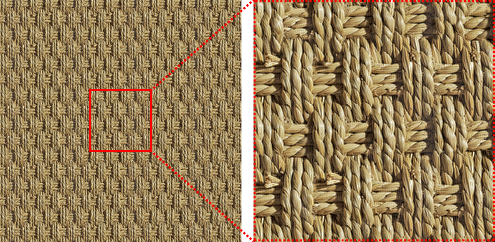}
    \caption[]{\bacon{} periodic extrapolation behavior\protect\footnotemark.}
    \label{fig:extrapolation}
    \vspace{-1.5em}
\end{figure}

Finally, we derive a principled initialization scheme that ensures the distribution of activation functions at the output of each layer is distributed uniformly at the beginning of training.
\red{While the proposed scheme and that of \textsc{Siren}~\cite{sitzmann2020siren} both involve sine non-linearities, our initialization explicitly accounts for Hadamard products in the architecture and the distribution of inputs to sine layers $g_i$.}
%\red{Similar to \textsc{SIREN}~\cite{sitzmann2020siren}, our scheme involves sine non-linearities and seeks to maintain the distribution of layer activations throughout the network.
%Differently, we account for Hadamard products in the architecture and uniform initializations of inputs to the sine layers.}
We compare our initialization scheme to the initialization proposed by Fathony et al.~\cite{fathony2020multiplicative} in Fig.~\ref{fig:initialization}. 
Our proposed scheme resolves a problem with vanishingly small activations for deep networks and results in standard normal distributed activations after each linear layer.

In the supplemental, we provide an extended derivation, which we summarize as follows. 
Assume the input to the network is uniformly distributed $\mathbf{x}\sim\mathcal{U}(-0.5, 0.5)$, with $\boldsymbol{\omega}_i\sim\mathcal{U}(-B_i, B_i)$ and $\boldsymbol{\phi}_i\sim\mathcal{U}(-\pi, \pi)$, where we describe the distribution of each element of the matrix or vector.
Then, $\boldsymbol{\omega}_i\mathbf{x} + \boldsymbol{\phi}_i$ is distributed as
\begin{equation}
    \begin{dcases} 1/B_i\log\left(B_i/\min(|2x|,B_i) \right), & -B/2\leq x \leq B/2\\ 0 & \text{else} \end{dcases}\nonumber
\end{equation}
and $g_i(\mathbf{x}) = \sin(\boldsymbol{\omega}_i\mathbf{x} + \boldsymbol{\phi}_i)$ is approximately arcsine distributed with variance 0.5 (see supplemental, red plots in Fig.~\ref{fig:initialization}). 
Now, let $\mathbf{W}_i \sim \mathcal{U}[-\sqrt{6/d_\text{h}},\sqrt{6/d_\text{h}}]$.
Then we have that $\mathbf{W}_1 g_0(\mathbf{x}) + \mathbf{b}_1$ converges to the standard normal distribution with increasing $d_\text{h}$ (see supplemental).
Finally, the Hadamard product $g_1(\mathbf{x})\circ (\mathbf{W}_{1} \mathbf{z}_0 + \mathbf{b}_1)$ is the product of arcsine distributed and standard normal random variables which again has a variance of 0.5.
Applying the next linear layer results in another standard normal distribution, which is also the case after all subsequent linear layers (see red plots of Fig.~\ref{fig:initialization}).

\section{Experiments}
\label{sec:results}
We demonstrate \bacon{} on three separate tasks: image fitting, \red{view synthesis} using neural radiance fields, and 3D shape fitting using signed distance functions. 

\subsection{Image Fitting} \label{sec:img_fit}
We use an image fitting task to evaluate the performance of \bacon{} and to demonstrate its band-limited behavior.
\bacon{} is compared to three other baselines: a network with Gaussian Fourier Features positional encoding~\cite{tancik2020fourier}, \textsc{Siren}~\cite{sitzmann2020siren}, and the integrated positional encoding of Mip-NeRF~\cite{barron2021mip}, which is scale-dependent.

\footnotetext{\textit{Image:} \url{https://www.sketchuptextureclub.com/}}
We initialize all networks with 4 hidden layers, 256 hidden features, and we train on the 256$\times$256 resolution image for 5000 iterations using PyTorch~\cite{paszke2019pytorch} and Adam~\cite{kingma2014adam}.
The batch size is equal to the number of image pixels.
\red{For Fourier Features and integrated positional encoding, we use encoding scales of 6 and 10, respectively, to balance between image quality and high-frequency overfitting.}
For \textsc{Siren}, we initialize the frequency parameter to $\omega_0=30$.

\begin{figure*}[ht]
    \centering
    \includegraphics{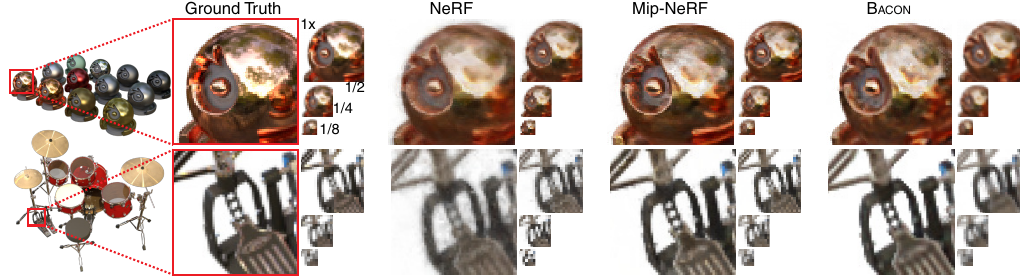}
    \caption{Neural rendering results. We compare NeRF~\cite{mildenhall2020nerf}, Mip-NeRF~\cite{barron2021mip}, and \bacon{} supervised on a multiscale synthetic dataset~\cite{barron2021mip}. \bacon{} captures higher frequency details better than NeRF while requiring fewer parameters to render at 1/2, 1/4, and 1/8 resolution.} 
    \label{fig:nerf}
    \vspace{-1.5em}
\end{figure*}

Fourier Features and \textsc{Siren} are trained to minimize the loss $\mathcal{L}_\text{img} = \lVert \mathbf{y} - \mathbf{y}_\text{GT} \rVert_2^2$, where $\mathbf{y}$ is the network output and $\mathbf{y}_\text{GT}$ are the image pixel values.
For \bacon{}, we sum this loss over all network outputs, with explicit supervision at all scales on the full-resolution image.
The integrated positional encoding network is supervised explicitly on anti-aliased image pixels at 1/4, 1/2, and full resolution, following Barron et al.~\cite{barron2021mip}.
Finally, we initialize \bacon{} to have a maximum bandwidth $B$ of  $0.5$ cycles/pixel, which is the Nyquist limit for the image.
The frequencies $\boldsymbol{\omega}_i$ are initialized so that the outputs $\mathbf{y}_1$, $\mathbf{y}_2$, and $\mathbf{y}_4$ are constrained to quarter, half, or full bandwidth.
That is, $B_0=B_1= B/8$, and $B_2=B_3=B_4=B/4$ such that $\sum_i B_i = B$.

Results of image fitting on a test scene from the Kodak dataset~\cite{kodak} are shown in Fig.~\ref{fig:image}.
Since Fourier Features and \textsc{Siren} only represent the signal at the trained resolution, sampling the network at 1/4 resolution results in aliasing.
We show the interpolation performance of these networks by evaluating a 4$\times$ upsampled grid of 1024$\times$1024 pixels.
\red{When upsampled, \bacon{} does not synthesize spurious high frequencies and has a band-limited output.
All other methods have non-zero high-frequency spectra and exhibit artifacts in the reconstruction. We show additional image fitting experiments in the supplemental, including evaluation of deep 8- and 16-layer \bacon{}s and MFNs.}

\vspace{-1.5em}
\paragraph{Periodic Extrapolation.}
Since \bacon{} uses discrete frequencies at each sine layer $g_i(\mathbf{x})$, the representation is periodic.
We demonstrate this by fitting a seamless texture using coordinates $\mathbf{x}\in[-0.5, 0.5]$ (red square of Fig.~\ref{fig:extrapolation}) and querying the network output for $\mathbf{x}\in[-2, 2]$.

\vspace{-2.5em}
\red{\paragraph{Scale Interpolation.}
Although \bacon{} outputs at discrete scales, we can interpolate between multiscale outputs, similar to the trilinear filtering used to render from mipmaps~\cite{williams1983pyramidal}. 
See supplemental for additional details and results.}

\subsection{Neural Radiance Fields}
Neural radiance fields (NeRF)~\cite{mildenhall2020nerf} have become a popular method for \red{view synthesis} and neural rendering.
The method operates on a dataset of multiview images with known camera positions, where each image pixel is associated with a ray $\mathbf{r}(t) = \mathbf{o} + t\mathbf{d}$ that extends from the camera center of projection $\mathbf{o}$ in the direction $\mathbf{d}$ passing through the pixel. 
A pixel color $\mathbf{C}(\mathbf{r})$ is predicted using the volume rendering equation to integrate predicted intermediate values of color $\mathbf{c}$ and opacity $\sigma$ along the ray~\cite{barron2021mip}. 
In practice, a neural network is queried to evaluate samples of $\mathbf{c}$ and $\sigma$ along each ray $\mathbf{r}(t)$, and the volume rendering integral is evaluated using quadrature as~\cite{Mildenhall:2019,max1995optical}
\vspace{-0.8em}
\begin{equation}
\begin{aligned}
    \mathbf{C}(\mathbf{r}, \mathbf{t}) = \sum\limits_j T_j(1 - \exp(-\sigma_j(t_{j+1} - t_j)))\,\mathbf{c}_j,\\
    \text{with} \quad T_j - \exp\left(-\sum\limits_{i'<i}\sigma_{i'}(t_{i'+1} - t_{i'}) \right),
\end{aligned}
\vspace{-0.5em}
\end{equation}
where $T_j$ represents the transmittance or visibility of a point on the ray, and the values $w_j = T_j(1 - \exp(-\sigma_j(t_{i+1} - t_j)))$ can be interpreted as alpha compositing weights applied to the predicted colors $\mathbf{c}_j$.
After training, novel views can be rendered by simply evaluating the corresponding rays.

\begin{table}[t]
    \centering
    \resizebox{\columnwidth}{!}{
    \begin{tabular}{lccccc|cccc}
    \toprule
    & \multicolumn{5}{c|}{PSNR $\uparrow$} & \multicolumn{4}{|c}{\# Params.}\\
                    & 1$\times$ & 1/2 & 1/4 & 1/8 & Avg. & 1$\times$ & 1/2 & 1/4 & 1/8 \\\midrule
        NeRF        & 26.734 & 28.941 & 29.297 & 26.464 & \multicolumn{1}{|c|}{27.859} & \multicolumn{4}{c}{\xrfill{0.5pt}511K\xrfill{0.5pt}} \\
        Mip-NeRF    & 29.874 & 31.307 & 32.093 & 32.832 & \multicolumn{1}{|c|}{\textbf{31.526}} & \multicolumn{4}{c}{\xrfill{0.5pt}511K\xrfill{0.5pt}} \\
        \bacon{}    & 27.430 & 28.066 & 28.520 & 28.475 & \multicolumn{1}{|c|}{\underline{28.123}} & 531K & 398K & 266K & 133K \\
    \bottomrule
    \end{tabular}}
    \caption{Performance of NeRF, Mip-NeRF, and \bacon{} averaged across the multiscale Blender dataset. \bacon{} achieves better average performance than NeRF while requiring fewer parameters to render the lower resolution images.}
    \label{tab:nerf}
    \vspace{-1.5em}
\end{table}

\begin{figure*}[ht]
    \centering
    \includegraphics{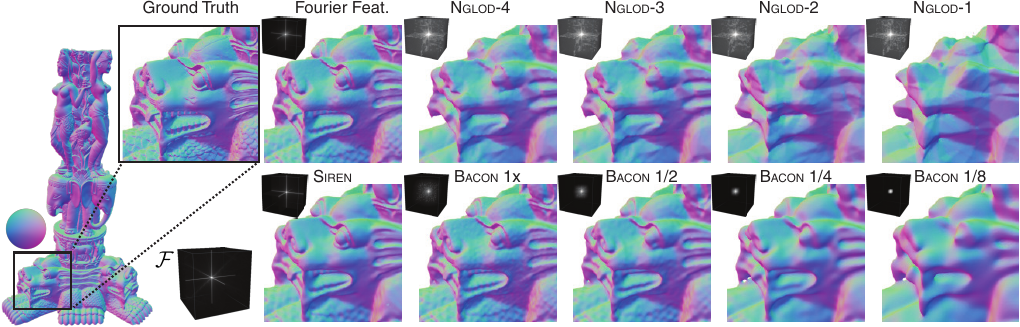}
    \caption{Shape fitting results. Results on the Thai Statue from the Stanford 3D Scanning Repository are shown for levels-of-detail 1--4 of Neural Geometric Level of Detail (\nglod{})~\cite{takikawa2021nglod}, Fourier Features~\cite{tancik2020fourier}, \textsc{Siren}~\cite{sitzmann2020siren}, and \bacon{}. All methods perform similarly at their highest detail output (see Table~\ref{tab:shapes}), but \bacon{} learns a smooth multiscale decomposition of the shape. Insets show the spectra of the extracted signed-distance functions, revealing the band-limited output of \bacon{}. Additional results included in the supplemental.}
    \label{fig:shapes}
    \vspace{-1.5em}
\end{figure*}

We evaluate \bacon{} for this task and compare to NeRF and Mip-NeRF baselines trained on a multiscale Blender dataset~\cite{barron2021mip} with images at full (512$\times$512), 1/2, 1/4, and 1/8 resolution.
For the baselines, we use the implementations of Barron et al.\footnote{\url{https://github.com/google/mipnerf}}~\cite{barron2021mip}. All networks are trained according to the procedure of Mip-NeRF; we use the Adam optimizer with a batch size of 4096 rays and 1e6 training iterations.
The learning rate is annealed logarithmically from 1e-3 to 5e-6 for \bacon{} and 5e-4 to 5e-6 for the baselines.
All networks are composed of 8 hidden layers with 256 hidden features.

For \bacon{}, we adapt the training procedure and architecture as follows.
Rays within the multiscale Blender dataset fall within an 8 by 8 unit volume ($\mathbf{r}(t)\in[-4, 4]^3$), and we find that setting the maximum bandwidth $B$ to 64 cycles per unit interval allows fitting high frequency image details.
\red{To simplify the training procedure, we evaluate all methods without the viewing direction input originally used for NeRF.
This also enables visualization of the \bacon{} Fourier spectrum (see Fig.~\ref{fig:teaser}).}
Thus the input to all networks is a 3D coordinate corresponding to the position along the ray $\mathbf{r}(t)$.
\bacon{} produces four outputs, one for each scale of the dataset: $\mathbf{y}_i$, $i \in [2, 4, 6, 8]$.
The $B_i$ constrain each output to 1/8, 1/4, 1/2, and full resolution, with $\sum_{i=0}^2 B_i = B/8$, $\sum_{i=0}^4 B_i = B/4$, and so on. 
We also adapt the hierarchical sampling procedure of NeRF~\cite{mildenhall2020nerf}, wherein the alpha compositing weights $w_j$ from an initial forward pass are used to resample the ray in regions of non-zero opacity.
To improve efficiency, we use the lowest-resolution output of the network for this initial forward pass and apply the following loss function on pixels rendered using the resampled rays with 256 samples.
\vspace{-0.5em}
\begin{equation}
    \mathcal{L}_\text{\bacon{}} = \sum\limits_{i, j, k} \lVert (\mathbf{C}_k(\mathbf{r}_i, \mathbf{t}_j) - \mathbf{C}_{\text{GT},k}(\mathbf{r}_i) \rVert^2_2,
    \vspace{-0.5em}
\end{equation}
where $i$, $j$, and $k$ index rays, ray positions, and dataset scales, respectively.
\red{For quantitative evaluation, we use per-scale supervision so the \bacon{} outputs are directly comparable to the multiscale ground truth images.}
Finally, we adopt the regularization strategy of Hedman et al.~\cite{hedman2021snerg} to penalize non-zero off-surface opacity \red{(see supplemental for results without per-scale supervision and an ablation study)}. 

Qualitative and quantitative evaluations of \bacon{} for neural rendering, are shown in Fig.~\ref{fig:nerf} and Table~\ref{tab:nerf}. 
\red{\bacon{} achieves better performance than NeRF trained on the multiscale dataset at 1/8 and 1$\times$ resolution}.
We report PSNR at each scale, averaged over all scenes in the multiscale Blender dataset in Table~\ref{tab:nerf}.
In Fig.~\ref{fig:nerf}, we observe that \bacon{} recovers higher frequency details compared to NeRF on the \textit{Materials} and \textit{Drums} scenes.
Mip-NeRF incorporates an additional mechanism which changes the positional encoding along each ray to account for the expansion of the viewing frustum, and achieves the best performance. 
Still, we find that \bacon{} produces high-quality results with a fraction of the parameters at low resolution. 

Additionally, we can use \bacon{} to learn semi-supervised multiscale decompositions of the neural radiance fields.
In this case, we train each output scale at the full resolution, and \bacon{} automatically learns band-limited representations at the intermediate output layers. 
We show an example of this for the \textit{Lego} scene in Fig.~\ref{fig:teaser}. 
Additional results for \bacon{} in the explicitly supervised and semi-supervised cases are shown in the supplemental.

\begin{table}[t]
    \centering
    \resizebox{\columnwidth}{!}{
    \begin{tabular}{lccccc}
    \toprule
    & FF & \textsc{Siren} & \nglod{}-4 & \nglod{}-5 & \bacon{} 1$\times$ \\\midrule
        \# Params. & 527K & 528K & 1.35M & 10.1M & 531K \\
        Chamfer$\downarrow$ & \textbf{2.166e-6} & 2.780e-6 & 8.358e-6 & 2.422e-6 & \underline{2.198e-6} \\
        IOU $\uparrow$ & \textbf{9.841e-1} & 9.751e-1 & 9.479e-1 & 9.811e-1 & \underline{9.833e-1}\\\bottomrule
\end{tabular}}
    \caption{Shape fitting performance of Fourier Features~\cite{tancik2020fourier}, \textsc{Siren}~\cite{sitzmann2020siren}, Neural Geometric Level of Detail (\nglod{})~\cite{takikawa2021nglod}, and \bacon{} averaged across 5 test scenes (detailed in main text). All methods achieve roughly comparable performance, including \bacon{} despite simultaneously representing multiple scales. Multiple levels of detail are shown for \nglod{}, which requires more parameters to populate the explicit feature grids.}
    \label{tab:shapes}
    \vspace{-1.5em}
\end{table}

\subsection{3D Shape Representation} \label{sec:shapes}
Neural representation networks have shown promise for representing and manipulating 3D shapes.
\bacon{} is well-suited for this task, and we evaluate its performance on a range of shapes from the Stanford 3D scanning repository\footnote{\url{http://graphics.stanford.edu/data/3Dscanrep/}}.

We compare \bacon{} to Fourier Features~\cite{tancik2020fourier}, \textsc{Siren}~\cite{sitzmann2020siren}, and Neural Geometric Level of Detail (\nglod{})~\cite{takikawa2021nglod}, a representation which optimizes explicit features stored on a sparse voxel octree.
\red{We do not compare to Mip-NeRF since it requires per-scale supervision.}
All networks are trained to directly fit a signed distance function (SDF) estimated from a ground truth mesh.

For \bacon, Fourier Features, and \siren, we use networks with 8 hidden layers and 256 hidden features.
For Fourier Features (Gaussian encoding), we set the encoding scale to 8 and for \textsc{Siren}, we set $\omega_0=30$. 
We train on locations sampled from the zero level set and add Laplacian noise; this results in an exponential decay in the number of samples off the zero level, as proposed by Davies et al.~\cite{davies2020effectiveness}.
We find that the width of the Laplacian distribution has a large impact on performance.
Setting the variance $\sigma_\text{L}^2$ too small results in poor off-surface fitting, but setting the variance too high reduces the number of samples on the zero level set, degrading the appearance of the surface.
Thus, we introduce a coarse and fine sampling procedure wherein we produce ``fine'' samples using a small variance of  $\sigma^2_\text{L}=\text{2e-6}$ and ``coarse'' samples with $\sigma_\text{L}^2=\text{2e-2}$.
Samples are drawn in the domain $[-0.5, 0.5]^3$, and we initialize the frequencies of \bacon{} similar to the NeRF experiments (additional details in supplemental).
We train using a loss function
\vspace{-0.5em}
\begin{equation}
    \mathcal{L}_\text{SDF} = \lambda_\text{SDF} \lVert \mathbf{y}^c - \mathbf{y}^c_\text{GT}\rVert_2^2 + \lVert \mathbf{y}^f - \mathbf{y}^f_\text{GT} \rVert_2^2,
    \vspace{-0.5em}
\end{equation}
where $\mathbf{y}$ is the network output, $\mathbf{y}_\text{GT}$ represents the ground truth SDF values, the $f$ and $c$ superscripts indicate fine and coarse samples, and \red{we set $\lambda_\text{SDF}$ to 0.01 for all experiments.}
\red{For \bacon{} we compute this loss at all output scales.}

We train Fourier Features, \textsc{Siren}, and \bacon{} on each dataset for 200,000 iterations with a batch size of 5,000 coarse and 5,000 fine SDF samples. 
Models are optimized using Adam~\cite{kingma2014adam}, and we logarithmically anneal the learning rate of each method from 1e-2 (\bacon{}), 1e-3 (Fourier Features), and 1e-4 (\textsc{Siren}) to a final value of 1e-4 during the course of training.
For \nglod{}, we use the default training settings in the authors' code\footnote{\url{https://github.com/nv-tlabs/nglod}}, which samples 500,000 points at each training epoch, uses a batch size of 512, and trains for 250 epochs.
We train \nglod{} models with a maximum of 4 or 5 levels of detail.
Additional levels of detail result in improved performance, but require more memory.

\begin{figure}[ht]
    \centering
    \includegraphics[width=\columnwidth]{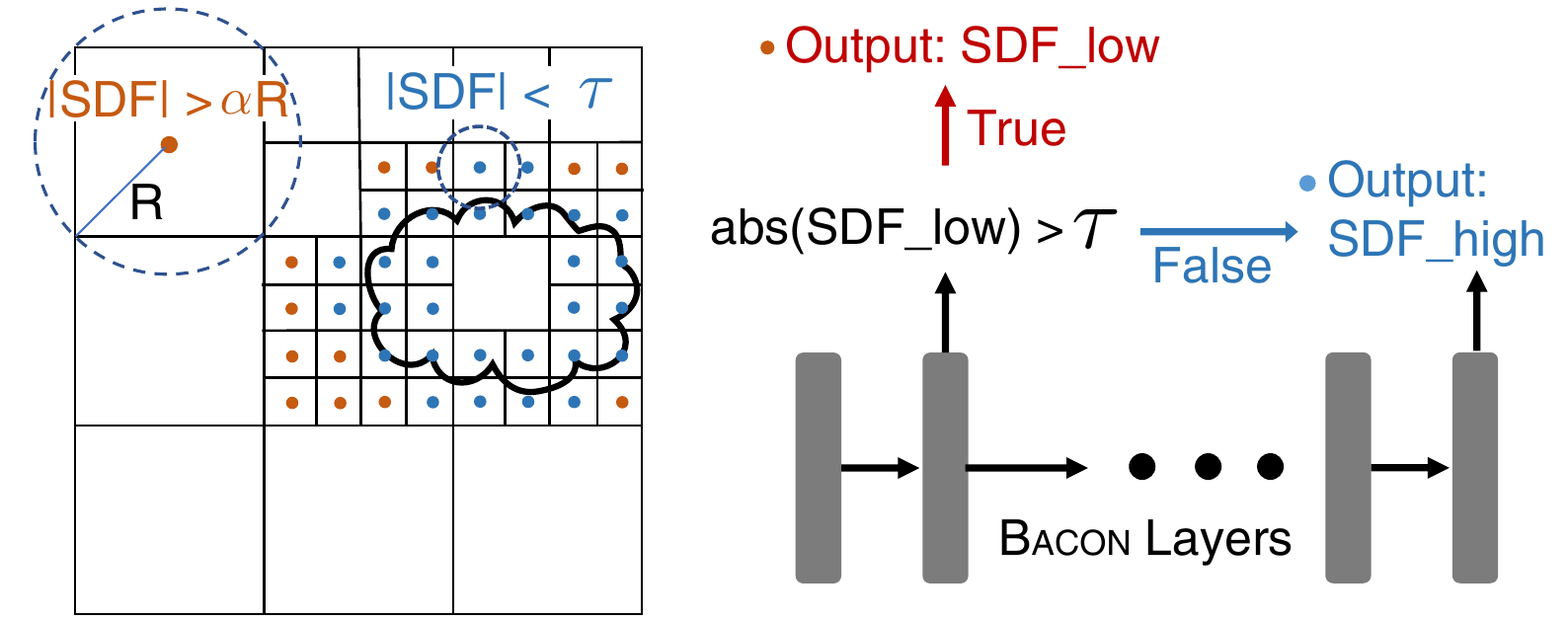}
    \vspace{-2em}
    \caption{Adaptive-frequency multiscale SDF evaluation for fast mesh extraction. 
    We propose a multi-scale evaluation which subdivides non-empty cells, i.e., when the SDF is smaller than the cell radius. We further accelerate the evaluation by using fewer layers (red) for regions that do not require high-frequency details. We evaluate the full network (blue) only when query locations are sufficiently close ($|\text{SDF}|<\tau$) to the surface.}
    \label{fig:SDF_MC}
    \vspace{-0.5em}
\end{figure}

\begin{table}[t]
    \centering
    \resizebox{\columnwidth}{!}{
    \begin{tabular}{lccc}
    \toprule
    & Dense Grid ($512^3$) & Adaptive-Frequency & Adaptive + Multiscale (Proposed) \\\midrule
        Time (s) & 17.91 & 5.50 & \textbf{0.222} \\
        \bottomrule
\end{tabular}}
\caption{SDF evaluation time. The proposed Adaptive + Multiscale method achieves a roughly 80$\times$ speedup over naive evaluation of the SDF on a dense grid (averaged over 5 test scenes).}
    \label{tab:MC}
    \vspace{-1.8em}
\end{table}

We fit each method to four scenes from the Stanford 3D Scanning Repository (\textit{Armadillo}, \textit{Dragon}, \textit{Lucy}, and \textit{Thai Statue}), as well as a simple sphere baseline (all objects are shown in the supplemental).
The models are extracted at $512^3$ resolution using marching cubes and evaluated using Chamfer distance and intersection over union (IOU), and we report these numbers averaged over the 5 scenes in Table~\ref{tab:shapes}. 
The highest resolution outputs of all methods achieve comparable performance, though note that \nglod{}-5 requires over an order of magnitude more parameters than the other representations, and \bacon{} achieves this despite representing all scales simultaneously.

We see similar qualitative trends in Fig.~\ref{fig:shapes}, with Fourier Features, \textsc{Siren}, \nglod{}-4, and \bacon{} all producing detailed reconstructions of the \textit{Thai Statue} scene.
\bacon{} produces a smooth reconstruction at multiple scales because of its band-limited output layers.
This can be compared to the low-resolution outputs of \nglod{}, which show fewer finer details, but also have coarse, angular artifacts from the ReLU non-linearity used in the network. 
This follows from the \nglod{} frequency spectrum (see Fig.~\ref{fig:shapes}), which is non-zero for high frequencies, including at the coarsest scale. 

\vspace{-1em}
\paragraph{Accelerated Marching Cubes.}
We observe that the band-limited, multi-output nature of our network allows efficient allocation of resources when evaluating SDFs on a dense grid for mesh extraction.
The key idea is to use the lower-layer output of \bacon{} when a cell is far away from the surface (Fig.~\ref{fig:SDF_MC}).
That is, we early-stop the computation within the network when $|\text{SDF}|<\tau$, where we set $\tau$ to 0.7$\times$ the finest voxel size. This adaptive computation significantly reduces the mesh extraction time (Table~\ref{tab:MC}).

Moreover, we propose a multiscale approach for further acceleration.
As SDF values indicate the distance to the closest surface, we can consider a cell to be empty when $|\text{SDF}|>\alpha\text{R}$, for circumsphere radius $\text{R}$ and some $\alpha>1$ that improves robustness to imperfect SDFs (we use $\alpha=2$).
Starting from the coarsest resolution grid, we subdivide a cell only when $|\text{SDF}|<\alpha\text{R}$ to prune empty space.
Multiscale extraction approaches have been proposed for extracting occupancy fields from coordinate networks~\cite{mescheder2019occupancy}, but using an SDF facilitates pruning since each sample reveals a region of empty space.

We combine the two strategies by applying the adaptive-frequency evaluation on each level of the multiscale grids, leading to roughly 80$\times$ faster SDF evaluation than the naive approach as shown in Table~\ref{tab:MC} (see supplemental results, all timings evaluated on an NVIDIA RTX A6000 GPU).

\section{Conclusion}
\label{sec:conclusion}
\vspace{-0.2em}
In this work, we take steps towards making coordinate networks interpretable and scale aware.
Our approach enables analyzing and controlling the spectral bandwidth of the network at intermediate layers, allowing multiscale signal representation, even without explicit supervision.  
\red{Since we can characterize the bandwidth of the network using Fourier analysis, its behavior is provably constrained, even at unsupervised locations.}
Moreover, \bacon{}'s intermediate outputs help to improve inference times via adaptive frequency evaluation.
We show that \bacon{} outperforms other single-scale coordinate networks for multiscale image fitting, neural rendering, and 3D scene representation. 

\vspace{-1em}
\paragraph{Limitations.}
We also highlight a few limitations of \bacon{} and promising future directions. 
% In this work, we demonstrated \bacon{} for representing two- and three-dimensional signals.
While we demonstrated fitting two- and three-dimensional signals, fitting signals in higher dimensions may require more parameters to achieve dense spectral coverage due to the curse of dimensionality.
Still, it may be possible to optimize the initialization of frequencies in a way that maximizes spectral coverage and mitigates this challenge.
Also, our current work is limited to single scene overfitting.
However, many generative models  work by increasing the frequency of the output using successive upsampling layers~\cite{karras2019style}, which is similar in spirit to our method.
Recent work on band-limited models for image synthesis has shown great promise~\cite{karras2021alias}, so applying \bacon{} for generative modeling is an exciting area for future research.

\vspace{-1em}
\paragraph{Societal Impact.} 
\red{We condemn the misuse of scene representation networks, including \bacon{}, for malicious deepfakes or spreading misinformation, and we emphasize the importance of research to thwart such efforts (see, e.g., Tewari et al.~\cite{tewari2020state} for a discussion of related strategies).}

\vspace{-0.5em}
\paragraph{Acknowledgments.}  
\red{This project was supported in part by a PECASE by the ARO, NSF award 1839974, Samsung GRO, and Stanford HAI.}

{\small
\bibliographystyle{ieee_fullname}
\bibliography{references}
}

\end{document}

\onecolumn

%%%%%%%%% title
\title{\textsc{Bacon}: Band-limited Coordinate Networks\\ for Multiscale Scene Representation\\\vspace{0.3em}Supplemental Material}

\author{David B. Lindell
\qquad Dave Van Veen
\qquad Jeong Joon Park
\qquad Gordon Wetzstein\\[0.25em]
Stanford University\\
{\small\url{http://computationalimaging.org/publications/bacon}}
}

\maketitle

{
\begin{spacing}{1.5} 
\hypersetup{linkcolor=black}
\tableofcontents
\end{spacing}
}

\newpage

\section{Supplemental Derivations}

\subsection{Deriving the Number of Sines}

\begin{lemma}
The product of two sines is the sum of two sines with frequencies corresponding to the sum and difference of the initial frequencies. 
\begin{align}
    \sin(\omega_1 x + \phi_1) \cdot \sin(\omega_2 x + \phi_2) = \frac{1}{2} \Big[\sin\big((\omega_1 + \omega_2)\cdot x + \phi_1 + \phi_2 - \pi/2\big) + \sin\big((\omega_1 - \omega_2)\cdot x + \phi_1 - \phi_2 + \pi/2\big)\Big]
\end{align}
\label{lemma:sine_product}
\end{lemma}
\begin{proof}
    Without loss of generality, let $\sin(a) = \sin(\omega_1 x + \phi_1)$ and $\sin(b) = \sin(\omega_2 x + \phi_2)$.
    Then, 
    \begin{align}
        \sin(a) \cdot \sin(b) &= \frac{e^{ja} - e^{-ja}}{2j} \cdot  \frac{e^{jb} - e^{-jb}}{2j} \\
                              &= \frac{e^{j(a+b)} + e^{-j(a+b)} - e^{j(a-b)} - e^{-j(a-b)}}{-4} \\
                              &= -\frac{1}{2}\frac{e^{-j\pi/2}}{-j}\frac{e^{j(a+b)} + e^{-j(a+b)}}{2} + \frac{1}{2}\frac{e^{j\pi/2}}{j}\frac{e^{j(a-b)} + e^{-j(a-b)}}{2} \\
                              &= \frac{1}{2}\frac{e^{j(a+b-\pi/2)} + e^{-j(a+b-\pi/2)-j\pi}}{2j} + \frac{1}{2}\frac{e^{j(a-b+\pi/2)} + e^{-j(a-b+\pi/2)+j\pi}}{2j} \\
                              &= \frac{1}{2}\frac{e^{j(a+b-\pi/2)} - e^{-j(a+b-\pi/2)}}{2j} + \frac{1}{2}\frac{e^{j(a-b+\pi/2)} - e^{-j(a-b+\pi/2)}}{2j} \\
                              &= \frac{1}{2}\Big[\sin(a+b - \pi/2) + \sin(a-b +\pi/2)\Big]\\
    \end{align}
\end{proof}

\begin{theorem}
    The number of sines represented by a multiplicative filter network is given as 
    \begin{align*}
        N_\text{\normalfont sine}^{(N_\text{\normalfont L})} = \sum\limits_{i=0}^{N_\text{\normalfont L}-1} 2^{i} d_\text{\normalfont h}^{i+1},
    \end{align*}
    where $N_\text{\normalfont L}$ is the number of layers and $d_\text{\normalfont h}$ is the number of hidden features in the network.
    \label{thm:num_sines}
\end{theorem}

\begin{proof}
Consider a $N_\text{L}=1$ layer network, given as 
\begin{align}
    \mathbf{z}_0 = \sin(\boldsymbol{\omega}_0 \mathbf{x} + \boldsymbol{\phi}_0)
\end{align}
This expression is a vector of sines, and the proof follows trivially for $N_\text{L}=1$; the number of sines represented by this network is exactly $d_\text{h}$ because we have $\mathbf{x} \in \mathbb{R}^{d_\text{in}}$ and $\boldsymbol{\omega}_0 \in \mathbb{R}^{d_\text{h} \times d_\text{in}}$.

To build intuition, we also analyze the case of $N_\text{L}=2$, which gives
\begin{align}
    \mathbf{z}_1 &= \sin(\boldsymbol{\omega}_1 \mathbf{x} + \boldsymbol{\phi}_1)\circ \left[\mathbf{W}_1 \sin(\boldsymbol{\omega}_0 \mathbf{x} + \boldsymbol{\phi}_0) + \mathbf{b}_1\right]\\
                 &= \sin(\boldsymbol{\omega}_1 \mathbf{x} + \boldsymbol{\phi}_1)\circ\left[\mathbf{W}_1 \sin(\boldsymbol{\omega}_0 \mathbf{x} + \boldsymbol{\phi}_0)\right] + \sin(\boldsymbol{\omega}_1 \mathbf{x} + \boldsymbol{\phi}_1)\circ \mathbf{b}_1\\
\end{align}
where
\begin{align}
    &\sin(\boldsymbol{\omega}_1 \mathbf{x} + \boldsymbol{\phi}_1)\circ\left[\mathbf{W}_1 \sin(\boldsymbol{\omega}_0 \mathbf{x} + \boldsymbol{\phi}_0)\right]\\
    &= \begin{bmatrix}
        \sin(\boldsymbol{\omega}_1^{(1)}\mathbf{x} + \phi_1^{(1)})\\
        \vdots\\
        \sin(\boldsymbol{\omega}_1^{(d_\text{h})}\mathbf{x} + \phi_1^{(d_\text{h})})
    \end{bmatrix}\circ
    \begin{bmatrix}
        W_1^{(1, 1)} \sin(\boldsymbol{\omega}_0^{(1)}\mathbf{x} + \phi_0^{(1)}) + \cdots + W_1^{(1, d_\text{h})} \sin(\boldsymbol{\omega}_0^{(d_\text{h})}\mathbf{x} + \phi_0^{(d_\text{h})}) \\
        \vdots \\
        W_1^{(d_\text{h}, 1)} \sin(\boldsymbol{\omega}_0^{(1)}\mathbf{x} + \phi_0^{(1)}) + \cdots + W_1^{(d_\text{h}, d_\text{h})} \sin(\boldsymbol{\omega}_0^{(d_\text{h})}\mathbf{x} + \phi_0^{(d_\text{h})}).
    \end{bmatrix}
\end{align}
Given that the Hadamard product between each row in the above expression results in a doubling of the number of sines (\reflemma{lemma:sine_product}), the entire Hadamard product results in a vector containing $2d_\text{h}^2$ sines.
The remaining term $\sin(\boldsymbol{\omega}_1 \mathbf{x} + \boldsymbol{\phi}_1)\circ \mathbf{b}_1$ contributes an additional $d_\text{h}$ sines, for a total $2d_\text{h}^2 + d_\text{h}$ sines in the $N_\text{L}=2$ layer network. In general, the linear component $\mathbf{W}_i$ of each layer multiplies the total number of sines by $d_\text{h}$ and the Hadamard product with the sine doubles this by a factor of 2. An additional $d_\text{h}$ sines are contributed by the bias term. 

Now, assume we have a network with $N_\text{L} = k$ layers. Let $L_i(\mathbf{z}) = \mathbf{W}_i \mathbf{z} + \mathbf{b}_i$ and let $g_i(\mathbf{x}) = \sin(\boldsymbol{\omega}_i\mathbf{x} + \boldsymbol{\phi}_i)$. Then we have
\begin{align}
    \mathbf{z}_{k-1} = g_{k-1}(\mathbf{x}) \circ (L_{k-1} (g_{k-2}(\mathbf{x}) \circ ( \ldots (\overbrace{g_2(\mathbf{x}) \circ (\overbrace{L_2 ( \underbrace{g_1(\mathbf{x}) \circ (\underbrace{L_1 \underbrace{g_0(\mathbf{x})}_{d_\text{h}}}_{d_\text{h}^2})}_{2d_\text{h}^2 + d_\text{h}})}^{d_\text{h}(2d_\text{h}^2 + d_\text{h})})}^{d_\text{h} + 2d_\text{h}(2d_\text{h}^2 + d_\text{h})})\ldots)))),
\end{align}
where the brackets indicate the number of sines for successive terms, revealing the following recursion for the number of sines in each layer.
\begin{align}
    N_\text{sine}^{(1)} &= d_\text{h} \\
    N_\text{sine}^{(i+1)} &= d_\text{h} + 2d_\text{h} N_\text{sine}^{(i)}
\end{align}

To complete a proof by induction let us assume that the theorem holds for $N_\text{L}=k$, and we have 

\begin{align}
    N_\text{sine}^{(k+1)} &= d_\text{h} + 2d_\text{h} N_\text{sine}^{(k)}\\
                          &= d_\text{h} + 2d_\text{h} \sum\limits_{i=0}^{k-1} 2^{i}d_\text{h}^{i+1} \\
            &= d_\text{h} + \sum\limits_{i=0}^{k-1} 2^{i+1}d_\text{h}^{i+2} \quad \text{let } j = i+1.\\
            &= d_\text{h} + \sum\limits_{j=1}^{k} 2^{j}d_\text{h}^{j+1}\\
            &= \sum\limits_{j=0}^{k} 2^{j}d_\text{h}^{j+1}
\end{align}
which is the original result, completing the proof.
\end{proof}

\begin{corollary}

If we remove the bias layers from each $L_i$, then we remove the additional of $d_\text{h}$ from the recursion and it becomes
\begin{align}
    \tilde{N}_\text{\normalfont sine}^{(1)} &= d_\text{h} \\
    \tilde{N}_\text{\normalfont sine}^{(i+1)} &= 2d_\text{h} N_\text{\normalfont sine}^{(i)},
\end{align}
such that 
\begin{align}
    \tilde{N}_\text{\normalfont sine}^{(N_\text{\normalfont L})} = 2^{N_\text{\normalfont L} - 1} d_\text{h}^{N_\text{\normalfont L}}.
\end{align}
\end{corollary}

\subsection{Deriving the Distribution of Frequencies}

\begin{lemma}
Let $X_i$, $i=1, 2, \ldots$  be a sequence of independent, identically distributed random variables. 
Then, let $N$ be a discrete random variable that is independent of $X_i$ and takes on values $N > 0$.
Now, define the compound random variable 
\begin{align}
    S_N = \sum\limits_{i=0}^{N-1} X_i,
\end{align}
The variance of $S_N$ is given by
\begin{align}
    \text{\normalfont Var}(S_N) = E[N]\, \text{\normalfont Var}(X_1) + \text{\normalfont Var}(N)E[X_1]^2
\end{align}
\label{lemma:total_variance}

\begin{proof}
    The proof follows from the law of total variance (see, e.g., p.\ 286 of Chatfield and Theobald~\cite{chatfield1973mixtures}).
\end{proof}
\end{lemma}

\begin{lemma}
Central Limit Theorem. Let $X_1, \ldots, X_n$ be a sequence of independent and identically distributed random variables with finite mean and variance, $\mu$ and $\sigma^2$, respectively, and let $S_n = \sum\limits_{i=1}^n X_i$.
For large $n$, $S_n$ approximates the normal distribution with $E[s_n] = n\mu$ and $\text{\normalfont Var}(s_n)=n\sigma^2$
    \begin{proof}
        See Ash et al.~\cite{ash2000probability}.
    \end{proof}
    \label{lemma:clt}
\end{lemma}

\begin{theorem}
    The frequencies of \bacon{} are approximately Gaussian distributed with variance equal to 
    \begin{align}
        \left[\sum\limits_{m=0}^{N_\text{L}-1} m \cdot \frac{2^{N_\text{L}-1-m}d_\text{h}^{N_\text{L}-m}}{\sum\limits_{i = 0}^{N_\text{L}-1} 2^{i}d_\text{h}^{i+1}}\right] \cdot \text{\normalfont Var}(\boldsymbol{\omega}_i^{(j_i)}). % $(k+1) \left[(2B_i + 1)^2 - 1 / 12 \right]$.
    \end{align}
where $\text{\normalfont Var}(\boldsymbol{\omega}_i^{(j_i)})$ is the variance of the initialized frequencies in each input sine layer $g_i$. 
    \begin{proof}
        The overall idea is that, if the number of sines is in the network is given as $\sum\limits_{i=0}^{N_\text{L}-1} 2^i d_\text{h}^{i+1}$ (\refthm{thm:num_sines}), then it turns out that the $i$th element in this sum describes the number of sines whose frequency is the sum of $i+1$ random variables.
        Then, we can use \reflemma{lemma:total_variance} and \reflemma{lemma:clt} to derive the variance and distribution of the network frequencies.

        First, let us show that the frequency of each sine represented by the network is itself a sum of random variables.
        We write an expression for the number of sines in the network $F_i(\boldsymbol{\omega})$ at frequency $\boldsymbol{\omega}$ directly after applying the Hadamard product with $g_i(\mathbf{x})$. Let $\delta(\boldsymbol{\omega})$ represent the Dirac delta function. Expanding on our previous results, we have that
\begin{align}
    F_0(\boldsymbol{\omega}) = \int\limits_{-\infty}^\infty \underbrace{\sum\limits_{j=0}^{d_\text{h}-1} \delta(\boldsymbol{\omega} - \boldsymbol{\omega}_0^{(j)})}_{f_0(\boldsymbol{\omega})} \,\mathrm{d}\boldsymbol{\omega}.
\end{align}
This expression simply places a delta function at the location of each frequency and the integral checks to see how many frequencies exist at the input parameter frequency, $\boldsymbol{\omega}$. 

\begin{figure*}[t!]
    \centering
    \includegraphics[width=\textwidth]{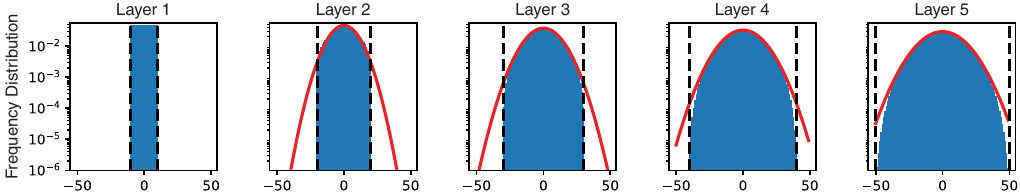}
    \caption{Empirical evaluation of distribution of frequencies. We plot the distribution of frequencies at the output of each layer in a 5-layer network with hidden features $d_\text{h} = 1024$ bandwidth $B_i$ equal to 10 (averaged over 1000 network realizations). The output of the first layer is uniformly distributed and successive layers increasingly approximate the normal distribution according to the Central Limit Theorem. A normal distribution with variance calculated using \refthm{thm:freqs} is plotted in red and closely matches the observed distribution.}
    \label{fig:freqs}
\end{figure*}
        
Now, recall the previous result that adding the next layer multiplies the number of frequencies by $2d_\text{h}$ and adds an additional $d_\text{h}$ frequencies. We use the convolution operator $\ast$ to shift the frequencies of the previous layer according to \reflemma{lemma:sine_product}. The additional $d_h$ frequencies result from the \red{Hadamard product} of the sine layer and the bias term from the previous layer. Thus we can give the following expression for the frequencies at layer $i+1$.
\begin{align}
    F_{i+1}(\boldsymbol{\omega}) = \int\limits_{-\infty}^\infty f_i(\boldsymbol{\omega}) \ast \underbrace{\sum\limits_{j=0}^{d_\text{h}-1} \left[\delta(\boldsymbol{\omega} - \boldsymbol{\omega}_{i+1}^{(j)}) + \delta(\boldsymbol{\omega} + \boldsymbol{\omega}_{i+1}^{(j)}) \right]}_{\text{shifts spectrum according to \reflemma{lemma:sine_product}}} + \underbrace{\sum\limits_{j=0}^{d_\text{h}-1} \delta(\boldsymbol{\omega} - \boldsymbol{\omega}_{i+1}^{(j)})}_{\text{new frequencies from the bias term}} \,\mathrm{d}\boldsymbol{\omega}.
\end{align}

If we ignore frequencies resulting from the bias terms in the above expression, we would have that there are $2^{N_\text{L}-1} d_\text{h}^{N_\text{L}}$ sines in an $N_\text{L}$ layer network (\refthm{thm:num_sines}) with output frequencies given as 
\begin{align}
    \boldsymbol{\bar{\omega}} = \boldsymbol{\omega}_0^{(j_0)} + s_1 \boldsymbol{\omega}_1^{(j_1)} + \cdots + s_{N_\text{L}-1} \boldsymbol{\omega}_{N_\text{L}-1}^{(j_{N_\text{L}-1})}
\end{align}
for some $s_0,\ldots,s_{N_\text{L}-1}\in\{-1, 1\}$ and indices $j_0, \ldots, j_{N_\text{L}-1} \in \{0, 1, \ldots, d_\text{h}-1\}$.
Including the bias term at the $i$th layer (for $i>0$) simply results in an additional $2^{N_\text{L} - 1 - i}d_\text{h}^{N_\text{L}- i}$ sines in the network, whose frequencies are a sum of $N_\text{L} - i$ terms (corresponding exactly to the terms of the sum in \refthm{thm:num_sines}).
Thus, frequencies represented by the network are drawn from a compound distribution (as in~\reflemma{lemma:total_variance}) since they can be the sum of from 1 to $N_\text{L}$ random variables.

Now we can describe the distribution of the frequencies of the network. 
Let the variance of an element of $\boldsymbol{\omega}_i$ be given by $\text{Var}(\boldsymbol{\omega}_i^{(j_i)})$.
Then, the output frequency $\boldsymbol{\bar{\omega}}$ is a compound random variable such that 
\begin{align}
    \boldsymbol{\bar{\omega}} = \boldsymbol{\omega}_M^{(j_M)} + \sum\limits_{i=M+1}^{N_\text{L}-1} s_i \boldsymbol{\omega}_i^{(j_i)} \quad M \in \{0, \ldots, N_\text{L}-1\}
\end{align}
and $M$ is a random variable whose probability depends on the total number of frequencies in the network that contribute to each possible value of $M$. 
Specifically, using \refthm{thm:num_sines} to evaluate the fraction of sines for each value of $M$ gives
\begin{align}
    p_M(m) = \frac{2^{N_\text{L}-1-m}d_\text{h}^{N_\text{L}-m}}{\sum\limits_{i = 0}^{N_\text{L}-1} 2^{i}d_\text{h}^{i+1}}.
\end{align}

We can calculate the resulting variance using the law of total variance outlined in \reflemma{lemma:total_variance}.

\begin{align}
    \text{\normalfont Var}(\boldsymbol{\bar{\omega}}) &= E[M]\, \text{\normalfont Var}(\boldsymbol{\omega}_i^{(j_i)}) + \text{\normalfont Var}(M)E[\boldsymbol{\omega}_i^{(j_i)}]^2\\
                                                      &= E[M]\, \text{\normalfont Var}(\boldsymbol{\omega}_i^{(j_i)}) & \boldsymbol{\omega}_i^{(j_i)} \text{ zero mean}\\
                                                      &= \sum\limits_{m=0}^{N_{\text{L}}-1}m\cdot p_M(m) \, \text{\normalfont Var}(\boldsymbol{\omega}_i^{(j_i)})\\
                                                      &= \left[\sum\limits_{m=0}^{N_\text{L}-1} m \cdot \frac{2^{N_\text{L}-1-m}d_\text{h}^{N_\text{L}-m}}{\sum\limits_{i = 0}^{N_\text{L}-1} 2^{i}d_\text{h}^{i+1}}\right] \cdot \text{\normalfont Var}(\boldsymbol{\omega}_i^{(j_i)})
\end{align}

Finally, we note that $\boldsymbol{\bar{\omega}}$ becomes the sum of a large number of random variables as the number of hidden layers in the network increases, and so we can approximate the distribution as a Gaussian using the Central Limit Theorem ~(\reflemma{lemma:clt}). We show that this holds empirically in Fig.~\ref{fig:freqs}, where we show the simulated distribution of frequencies in each layer of a 5-layer network with $d_\text{h}=1024$ and $B_i=10$, averaged over 1000 realizations. The distribution of frequencies is well-approximated by a Gaussian with the derived variance (especially for increasing layers).
\end{proof}
\label{thm:freqs}
\end{theorem}

\subsection{Initialization and Distribution of Activations}

\subsubsection{Preliminary Derivations}

\begin{lemma}
    Let $X$ and $Y$ be two independent random variables with probability density functions $f_X$ and $f_Y$. Then the probability density function $f_Z(z)$ of $Z=XY$ is given as 
   \begin{align}
       f_Z(z) = \int_{-\infty}^\infty f_X(x)f_Y(z/x)\frac{1}{|x|} \mathrm{d}x. 
   \end{align} 
   \begin{proof}
       See Grimmett and Stirzaker~\cite{grimmett2020probability}. 
   \end{proof}
   \label{lemma:pdf_prod}
\end{lemma}

\begin{theorem}
    Let W, X, and P be independent random variables sampled from continuous uniform distributions as  
    \begin{align}
        W &\sim \mathcal{U}(-B, B)\\
        X &\sim \mathcal{U}(-0.5, 0.5)\\
        P &\sim \mathcal{U}(-\pi, \pi) 
    \end{align}
    where $B \gg \pi$.
    Then let $Z = WX + P$. The probability density function $f_Z(z)$ of $Z$ is approximately
    \begin{align}
        \normalfont
        f_Z(z) \approx \begin{dcases} \frac{1}{B} \log\left(\frac{B}{|2z|} \right), & -B/2 \leq z \leq B/2\\ 0, & \text{else} \end{dcases}.
    \end{align}
    \begin{proof}
        Let $\tilde{Z} = WX$. Then, $f_{\tilde{Z}}(z)$ is given as (\reflemma{lemma:pdf_prod}):
        \begin{align}
            f_{\tilde{Z}}(z) &= \int_{-\infty}^{\infty} f_W(w) f_X(z/w)\,\frac{1}{|w|}\, \mathrm{d}w\\
            &= 2 \int_{0}^{\infty} f_W(w) f_X(z/w)\,\frac{1}{w}\, \mathrm{d}w\\
            &= \frac{1}{B}  \int_{0}^{B}f_X(z/w)\,\frac{1}{w}\, \mathrm{d}w\\
            &\quad f_X(z/w) = \begin{dcases} 1 & -0.5 \leq z/w \leq 0.5\\ 0 & \text{else} \end{dcases}\\
            &\quad\quad\quad\quad\;\;\:= \begin{dcases} 1 & w\leq -2z, \, w \geq 2z\\ 0 & \text{else} \end{dcases}\\
            &=\frac{1}{B} \int_{\text{min}(2z, B)}^B \frac{1}{w}\, \mathrm{d}w\\
            &= \frac{1}{B} \log(|w|)\Big\rvert_{\text{min}(2z, B)}^B\\
            &= \frac{1}{B} \log\left(\frac{B}{\text{min}(|2z|, B)} \right)\\
            &= \begin{dcases} \frac{1}{B} \log\left(\frac{B}{|2z|} \right), & -B/2 \leq z \leq B/2\\ 0, & \text{else} \end{dcases}
        \end{align}
        Now, $Z = WX + P = \tilde{Z} + P$ and $f_Z = f_{\tilde{Z}} \ast f_P$, where $\ast$ indicates convolution. For $B \gg \pi$, the support of $f_P$ is sufficiently small that we can neglect the ``broadening'' effect of the convolution, such that $f_Z \approx f_{\tilde{Z}}$.
    \end{proof}
    \label{thm:z_dist}
\end{theorem}

\begin{theorem}
    With $Z$ and $B$ as defined in \refthm{thm:z_dist}, we have that 
\begin{align}
    \normalfont\text{Var}[\sin(Z)] \approx \frac{1}{2}\left[1 - \frac{\text{SI}(B)}{B} \right] \approx \frac{1}{2}
\end{align}

    \begin{proof}
        \begin{align}
         \text{Var}[\sin(Z)] &= \text{E}[\sin^2(Z)] \\
         &= \frac{1}{2}(1 - \text{E}[\cos(2Z)])\\
         \text{E}[\cos(2Z)] &= \int_{-\infty}^{\infty} f_Z(z) \cos(2z)\, \mathrm{d}z\\
                            &\approx\frac{1}{B}\int_{-B/2}^{B/2} \log\left(\frac{B}{|2z|}\right) \cos(2z)\, \mathrm{d}z\quad \quad (\refthm{thm:z_dist})\\
         &\quad\text{Integrate by parts: } \int f\, \mathrm{d}g = fg - \int \mathrm{d}f g\\
         &\quad f=\log\left(\frac{B}{|2z|}\right), \, \mathrm{d}g = \cos(2z)\, \mathrm{d}z, \, \mathrm{d}f = -\frac{1}{z}\,\mathrm{d}z, \, g = \frac{1}{2}\sin(2z) \\ 
         &= \frac{1}{B}\left[\frac{1}{2} \log\left(\frac{B}{|2z|} \right) \sin(2z) + \underbrace{\int \frac{\sin(2z)}{2z}\,\mathrm{d}z}_{\frac{1}{2}\text{SI}(2z)}\right]_{-B/2}^{B/2} \\
         &= \frac{1}{2B}\left[\log\left(\frac{B}{|2z|} \right) \sin(2z) + \text{SI}(2z)\right]_{-B/2}^{B/2} \\
         &= \frac{1}{2B}\left[\log\left(\frac{B}{|2z|} \right) \sin(2z) + \text{SI}(2z)\right]_{-B/2}^{B/2} \\
         &= \frac{\text{SI}(B)}{B} \approx \frac{\pi}{2B}, \, B\gg 0\\
         \Rightarrow \frac{1}{2}(1 - \text{E}[\cos(2Z)]) &= \frac{1}{2}\left(1 - \frac{\text{SI}(B)}{B}\right)\\
                                                         &\approx \frac{1}{2}\left(1 - \frac{\pi}{2B}\right), \, B\gg 0\\
                                                         &\approx \frac{1}{2}
        \end{align}
        Where we used that $\text{SI}(x)$ is the sine integral function: $\normalfont\text{SI}(x) = \int_0^x \frac{\sin(t)}{t}\, \mathrm{d}t$.
    \end{proof}
    \label{thm:var_sine}
\end{theorem}

\begin{lemma}
The variance of the product of two random variables $X$ and $Y$ is given by 
\begin{align}
    \normalfont
    \text{Var}[X \cdot Y] = \text{Var}[X] \cdot\text{Var}[Y] + \text{E}[Y]^2 \cdot \text{Var}[X] + \text{E}[X]^2\cdot \text{Var}[Y]
\end{align}
\begin{proof}
    Refer to Goodman~\cite{goodman1960exact}.     
\end{proof}
\label{lemma:var_prod} 
\end{lemma}

\subsubsection{Proof of the Initialization Scheme}
\begin{theorem}
    Let the input to \bacon{} be uniformly distributed in $[-0.5, 0.5]$ and the frequency $\boldsymbol{\omega}_i$ of each layer be uniformly distributed in $[-B_i, B_i]$ with $B_i \gg 0$. Then, let the linear layer weights $\mathbf{W}_i$ applied after the sine layers be distributed according to a random uniform distribution in the interval $[-\sqrt{6 / d_\text{h}}, \sqrt{6/ d_\text{h}}]$. The activations after each linear layer are standard normal distributed.

\begin{proof}
    A sketch of the proof is as follows.
    \begin{itemize}
        \item The variance of the output of each sine layer $g_i(\mathbf{x})$ is approximately 0.5 (\refthm{thm:var_sine}).
        \item The output after applying the first linear layer $\mathbf{W}_1 g_0(\mathbf{x})$ is standard normal distributed (neglecting the effect of the bias).
            We have that the $i$th output is $\sum_{j=0}^{d_{\text{h}-1}} \mathbf{W}_1^{(i,j)} g_0(\mathbf{x})^{(j)}$, with $\text{Var}\left(\mathbf{W}_1^{(i,j)} g_0(\mathbf{x})^{(j)}\right) = \text{Var}\left(\mathbf{W}_1^{(i,j)}\right) \cdot \text{Var}\left(g_0(\mathbf{x})^{(j)}\right) = \frac{1}{12}\left(2\sqrt{6/d_\text{h}}\right)^2 \cdot \frac{1}{2} = 1/d_\text{h}$ (\reflemma{lemma:var_prod}). Then the entire sum has variance $d_\text{h}\cdot 1/d_\text{h} = 1$, and is normal distributed by the Central Limit Theorem (\reflemma{lemma:clt}).
        \item The output of the Hadamard product $g_{1}(\mathbf{x}) \circ (\mathbf{W}_{1} g_0(\mathbf{x}) + \mathbf{b}_1)$ has variance $\text{Var}(g_{1}(\mathbf{x})) \cdot \text{Var}(\mathbf{W}_{1}g_0(\mathbf{x})  + \mathbf{b}_1) \approx \frac{1}{2} \cdot 1 = \frac{1}{2}$ (\reflemma{lemma:var_prod}).
        \item The distribution after applying the next linear layer $\mathbf{W}_{i+1}$ is again standard normal (using the same steps of taking the product of the variances and applying the Central Limit Theorem), and the same steps as above can be repeated to show the distributions are standard normal after each linear layer.
    \end{itemize}
    \label{thm:initialization}
\end{proof}

\end{theorem}

\begin{figure*}[t!]
    \centering
    \includegraphics[width=\textwidth]{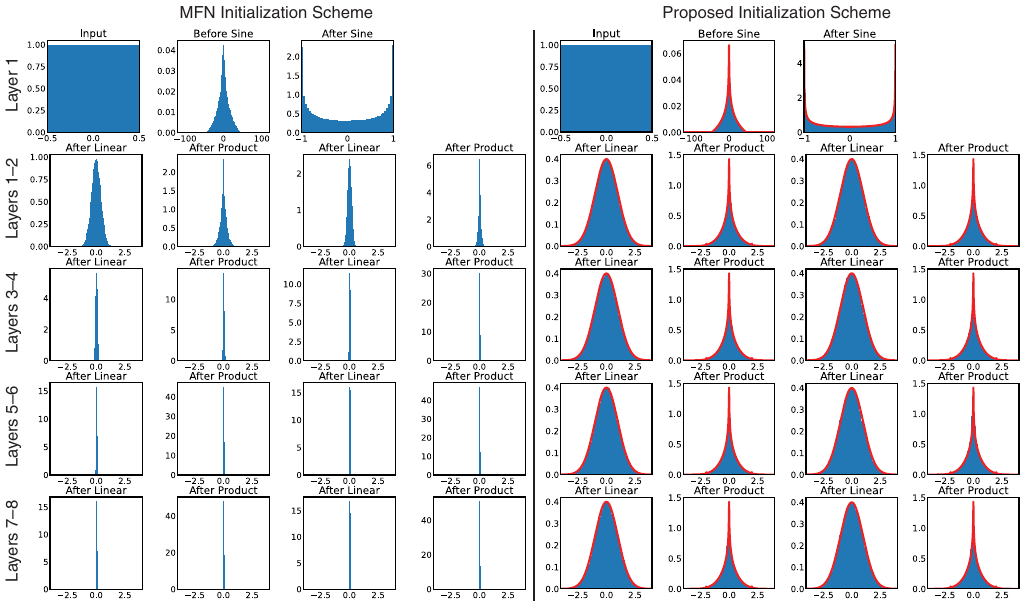}
    \caption{Empirical evaluation of initialization scheme. We show the default MFN initialization scheme (left)~\cite{fathony2020multiplicative}, and the proposed initialization scheme (right) for a network with 9 sine layers and 1024 hidden features ($d_\text{h}$). For our method we set the bandwidth $B_i$ of each layer to an arbitrary value of $30\pi$. We use the default settings of the MFN codebase, which initializes the linear layers $\mathbf{W}_i$ uniformly in $[-\sqrt{256/d_\text{h}},\sqrt{256/d_\text{h}}]$ and sets the bandwidth to $B_i = 256\sqrt{d_\text{in}}$.  Note that the proposed initialization scheme resolves the issue of the activation magnitude becoming extremely small with increased network depth. We also show that the distribution of activations closely matches our derivations, with the analytical expressions plotted in red.}
    \label{fig:activations}
\end{figure*}

\subsubsection{Empirical Evaluation}
We empirically evaluate the initialization scheme and derivations by showing plots of the distributions of activations for \bacon{} with 9 sine layers, 1024 hidden features ($d_\text{h}$), and an arbitrary bandwidth $B_i = 30\pi$. The plot is shown in Fig.~\ref{fig:activations}, and we overlay the analytical expression for the distribution at each intermediate output of the network.
We also compare to the conventional MFN initialization proposed by Fathony et al.~\cite{fathony2020multiplicative} using the publicly available implementation\footnote{\url{https://github.com/boschresearch/multiplicative-filter-networks}}.
Their proposed method initializes the linear layers $\mathbf{W}_i$ to a random uniform value in $[-\sqrt{256/d_\text{h}},\sqrt{256/d_\text{h}}]$ and sets the bandwidth to $B_i = 256\sqrt{d_\text{in}}$.
We find that this causes the magnitude of the activations to decrease significantly with increasing network size, leading to vanishing gradients.
There are five different distributions that can be observed at intermediate outputs of the network with our proposed initialization scheme:
\begin{enumerate}
    \item The input to the network is uniformly distributed in [-0.5, 0.5].
    \item The output of the linear layers applied directly to the input (that is, $\boldsymbol{\omega}_i \mathbf{x} + \boldsymbol{\phi}_i$) is distributed as derived in \refthm{thm:z_dist}.
    \item We find empirically that the output of $g_i(\mathbf{x}) = \sin(\boldsymbol{\omega}_i \mathbf{x} + \boldsymbol{\phi}_i)$ is approximately arcsine distributed and the variance is approximately $1/2$ as derived in \refthm{thm:var_sine}. In Fig.~\ref{fig:activations}, we show the empirical distribution together with a plot of the arcsine distribution with support over $[-1, 1]$ given by the probability density function $f_X(x) = \frac{1}{\pi \sqrt{1 - x^2}}$. Note the close correspondence of the arcsine distribution and the observed histogram of activations in Fig.~\ref{fig:activations}. We believe the connection to the arcsine distribution is related to previous work, which shows that the sine of uniform and normal random variables are both approximately arcsine distributed~\cite{sitzmann2020siren}. 
    \item The output of $\mathbf{W}_1 g_0(\mathbf{x}) + \mathbf{b}_1$ is standard normal distribution as described in \refthm{thm:initialization}. Also, the outputs of other linear layers $\mathbf{W}_i$ with bias $\mathbf{b}_i$ are standard normal distributed.
    \item The output of the Hadamard product $g_{i}(\mathbf{x}) \circ (\mathbf{W}_{i} \mathbf{z}_{i-1} + \mathbf{b}_i)$ is the product of an (approximately) arcsine distribution and a standard normal distribution. The distribution of a product of random variables is described by \reflemma{lemma:pdf_prod}, and the variance of this distribution is approximately equal to 1/2 (\reflemma{lemma:var_prod}). We numerically calculate the probability density function for the product of a standard normal and arcsine distribution according to \reflemma{lemma:pdf_prod} and find that this approximates the empirical distribution as shown in Fig.~\ref{fig:activations}.
\end{enumerate}

\newpage
\section{Supplemental Results}

\subsection{Images}
We evaluate \bacon{} on an image fitting task and compare its performance to three other methods: a ReLU network using Gaussian Fourier Features positional encoding (PE)~\cite{tancik2020fourier} and \siren{}~\cite{sitzmann2020siren}, both supervised at 256$\times$256 (1$\times$) resolution, and a ReLU network using integrated PE (adapted from Mip-NeRF)~\cite{barron2021mip} with supervision at 1/4, 1/2, and 1$\times$ resolutions.
\bacon{} is supervised at a single scale (1$\times$) and learns a multiscale decomposition.
All networks contain 4 hidden layers with 256 hidden features and are trained as described in the main text.

We perform a quantitative evaluation of the image fitting performance by training on a dataset of 16 randomly selected images from the DIV2K dataset and reporting the peak signal to noise ratio (PSNR) and structural similarity index measure (SSIM).
\red{We resize center crops of the images to 256$\times$256 resolution and then fit a model to the grid of pixels for each image.
We evaluate the performance on this training set of pixels as well as an offset validation grid of 256$\times$256 pixels whose values are bilinearly interpolated.}
Quantitative results in Table~\ref{tab:img_fit_quant} demonstrate that all methods fit the training set to well over 30 dB PSNR at at 256$\times$256 (1$\times$) resolution.
All methods perform similarly on the validation set.
Remarkably, \bacon{} demonstrates similar performance to the single-scale representations while simultaneously representing all output scales.

\red{Additional quantitative image results are shown in Table~\ref{tab:img_fit_multiscale}, evaluated across multiple scales.
Here, we train the network in the same fashion as above, but evaluate on a 1/4 (64 $\times$ 64) or 1/2 (128 $\times$ 128) resolution coordinate grid or a 4$\times$ upsampled grid (1024 $\times$ 1024).
We compare the network outputs to a bilinearly downsampled image or a high-resolution ground truth image in the case of 4$\times$ upsampling.}
\bacon{} and integrated positional encoding show the best performance for the low-resolution images while all methods perform similarly for upsampling.

Fig.~\ref{fig:supp_image} shows all output resolutions (1/4, 1/2, 1, and 4$\times$) for the result shown in the main text.
Fig.~\ref{fig:supp_image_dataset} shows additional results on center-cropped images from the DIV2K dataset~\cite{Agustsson_2017_CVPR_Workshops}.
In all results, Fourier Features and \siren{} fit to a single scale and show aliasing when subsampled, i.e.\ at 1/4 and 1/2 resolution.
Integrated PE learns reasonable anti-aliasing as it is explicitly supervised on anti-aliased pixel values.
The band-limited nature of \bacon{} allows it to closely represent a low-pass filtered image while only explicitly supervising at 1$\times$ resolution.
At 4$\times$ resolution, all methods except \bacon{} show high-frequency artifacts. 

\red{Fig.~\ref{fig:supp_antialias} shows an experiment where we compare \bacon{} and the network with a normal and low-pass filtered version of integrated position encoding at the 4$\times$ upsampled resolution.
The integrated positional encoding result contains spurious high-frequency details and artifacts from aliasing since the bandwidth of the network is not constrained.
Since aliasing corrupts the low-frequency components, these artifacts cannot be removed by applying a low-pass filter.}

\red{\paragraph{Deep \bacon{}.}
We compare deep 8- and 16-layer versions of \bacon{} with the proposed initialization scheme and MFNs with the original initialization scheme. 
For the image fitting task, we find that an 8-layer \bacon{} fits the lighthouse image shown in the main paper to 38.8 dB PSNR versus 29.8 dB PSNR for an 8-layer MFN.
For 16 layers, \bacon{} fits the image to 37.4 dB while the MFN architecture fails to optimize due to numerical instabilities.
We show convergence plots of the PSNR in Fig.~\ref{fig:supp_deep_mfn}.}

\red{\paragraph{Scale Interpolation.}
Interpolating between the discrete output scales allows a kind of continuous output scale to be achieved, similar to the trilinear filtering used to render from mimaps~\cite{williams1983pyramidal}.
To illustrate this effect, we sample \bacon{} at all output scales on the same 256$\times$256 resolution grid, and then linearly interpolate between resulting images.
Results are shown in Fig.~\ref{fig:supp_mipmap}.
Note that since linear interpolation is used, there is a discontinuous appearance of high-frequency Fourier coefficients when moving from one scale to the next.
Still, this simple technique allows blending between scales. 
}

\begin{table}[h]
    \centering
    \begin{tabular}{lc|cc|cc}
    \toprule
    & & \multicolumn{2}{c}{Training} & \multicolumn{2}{c}{Validation}\\
    & \# Params. & PSNR  &  SSIM  & PSNR & SSIM \\\midrule
        Fourier Features & 264K &  37.362 $\pm$ 2.544 & 0.976 $\pm$ 0.009                         & \textbf{29.771} $\pm$ 1.410 & \textbf{0.931} $\pm$ 0.022 \\
        \siren{}         & 265K &  \textbf{41.851} $\pm$ 2.084 & \textbf{0.987} $\pm$ 0.006       & 28.927 $\pm$ 1.756 & \underline{0.922} $\pm$ 0.034 \\
        Integrated PE        &  274K & 33.092 $\pm$ 2.219 & 0.930 $\pm$ 0.027                     & \underline{29.505} $\pm$ 1.498 & 0.901 $\pm$ 0.025 \\
        \bacon{}             &  268K & \underline{38.871} $\pm$ 1.727 & \underline{0.979} $\pm$ 0.005 & 29.266 $\pm$ 1.632 & \underline{0.922} $\pm$ 0.023 \\
    \bottomrule
    \end{tabular}
    \caption{Quantitative evaluation (mean $\pm$ standard deviation) for image fitting. For \bacon{} and Integrated Positional Encoding we compare to the highest resolution output.}
    \label{tab:img_fit_quant}
\end{table}

\begin{table}[h]
    \centering
    \begin{tabular}{l|lc|cc}
    \toprule
    Scale & Method & \# Params. & PSNR  &  SSIM \\\midrule
    \multirow{4}{*}{1/4}        & Fourier Features & 264K  & 24.484 $\pm$ 1.636 & 0.858 $\pm$ 0.055 \\
                                & \siren{}         & 265K  & 24.063 $\pm$ 1.832 & 0.843 $\pm$ 0.069 \\
                                & Integrated PE    & 274K  & \textbf{36.819} $\pm$ 1.697 & \textbf{0.984} $\pm$ 0.007 \\
                                & \bacon{}         & 67K   & \underline{31.179} $\pm$ 1.890 & \underline{0.948} $\pm$ 0.011 \\\midrule
    \multirow{4}{*}{1/2}        & Fourier Features & 264K  & 30.830 $\pm$ 1.462 & \underline{0.955} $\pm$ 0.017 \\
                                & \siren{}         & 265K  & 29.474 $\pm$ 2.024 & 0.942 $\pm$ 0.033 \\
                                & Integrated PE    & 274K  & \underline{33.020} $\pm$ 1.596 & \textbf{0.959} $\pm$ 0.013 \\
                                & \bacon{}         & 199K  & \textbf{33.140} $\pm$ 1.711 & \textbf{0.959} $\pm$ 0.009 \\\midrule
    \multirow{4}{*}{4$\times$}  & Fourier Features & 264K  & 25.909 $\pm$ 3.161 & 0.722 $\pm$ 0.122 \\
                                & \siren{}         & 265K  & \textbf{26.198} $\pm$ 3.255 & \textbf{0.740} $\pm$ 0.115 \\
                                & Integrated PE    & 274K  & 24.530 $\pm$ 2.477 & 0.667 $\pm$ 0.127 \\
                                & \bacon{}         & 268K  & \underline{25.967} $\pm$ 2.869 & \underline{0.731} $\pm$ 0.108 \\
    \bottomrule
    \end{tabular}
    \caption{Quantitative evaluation (mean $\pm$ standard deviation) for image fitting, evaluated at multiple scales.}
    \label{tab:img_fit_multiscale}
\end{table}

\begin{figure*}[t]
    \centering
    \includegraphics[width=\textwidth]{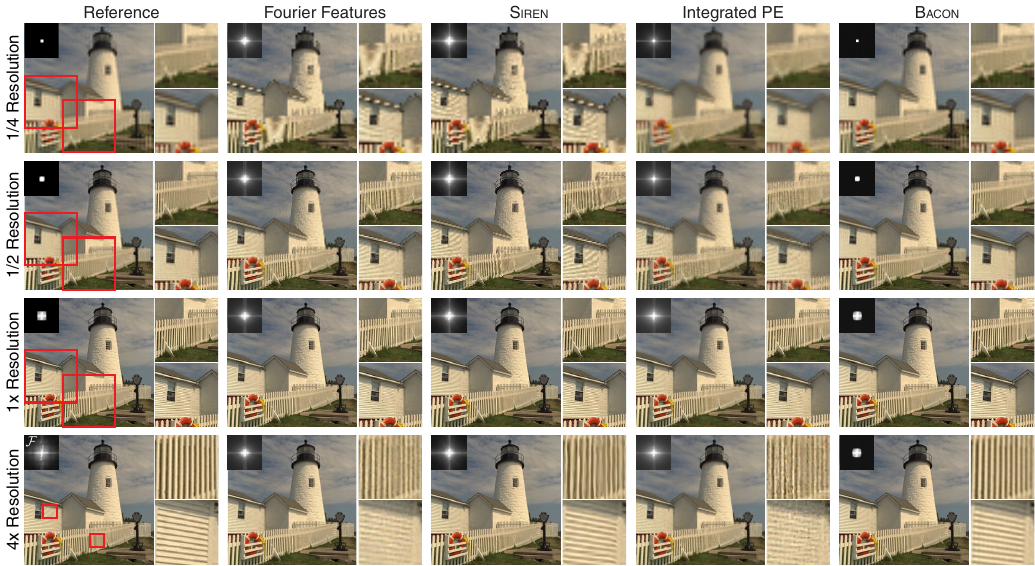}
    \caption{Image fitting results. We train a ReLU network using Fourier features (FF) positional encoding (PE)~\cite{tancik2020fourier} and a SIREN~\cite{sitzmann2020siren} to fit an image at 256$\times$256 (1$\times$) resolution, and evaluate the models at 64$\times$64 (1/4), 128$\times$128 (1/2), 256$\times$256 (1$\times$), and 1024$\times$1024 (4$\times$) resolution. Since these methods fit to a single scale, we see aliasing at lower resolutions, and high-frequency artifacts at 4$\times$ resolution (see insets). We train a ReLU network with integrated PE~\cite{barron2021mip} with supervision at 1/4, 1/2, and 1$\times$ resolutions. While this network learns anti-aliasing at low resolutions, inference at the unsupervised 4$\times$ resolution yields artifacts. Finally, \bacon{} is supervised at a single scale, learning band-limited outputs that closely match low-pass filtered reference images (see left column, and Fourier spectra insets). All methods achieve an accurate fit at 1$\times$ resolution with PSNRs of 37.838 dB (FF PE), 41.513 dB (SIREN), 34.105 dB (Integrated PE), and 40.314 dB (\bacon{}).}
    \label{fig:supp_image}
\end{figure*}

\begin{figure*}[t]
    \centering
    \includegraphics[width=\textwidth]{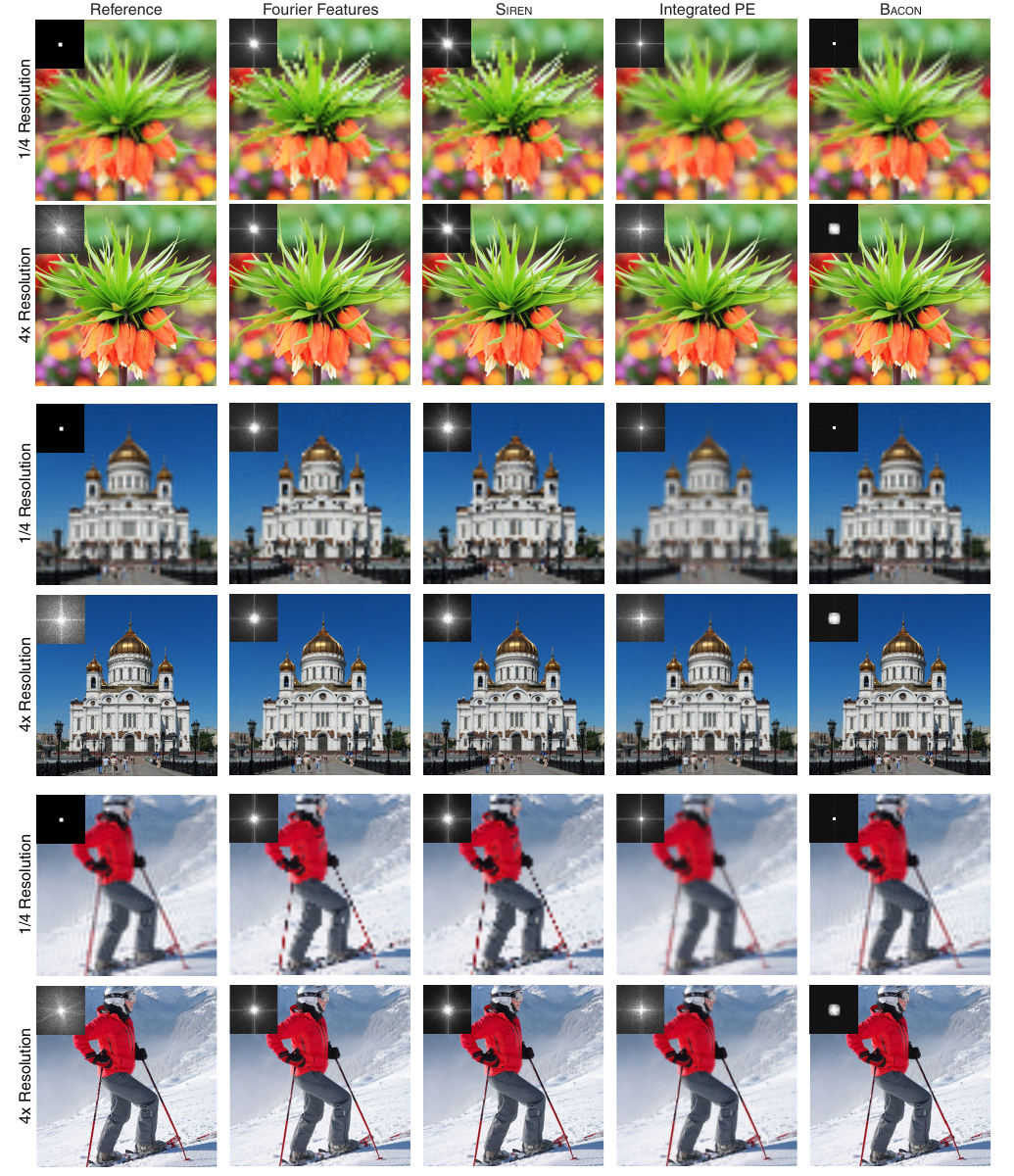}
    \caption{Supplemental image fitting results. We show the output of baseline methods and \bacon{} fit to a subset of images from the DIV2K dataset. Similar to previous results, we observe aliasing in baseline methods when subsampling to lower resolutions, and artifacts in 4$\times$ supersampled outputs. \bacon{} produces anti-aliased outputs at low-resolution and interpretable upsampled results via band-limited interpolation.}
    \label{fig:supp_image_dataset}
\end{figure*}

\begin{figure*}[t]
    \centering
    \includegraphics[width=\textwidth]{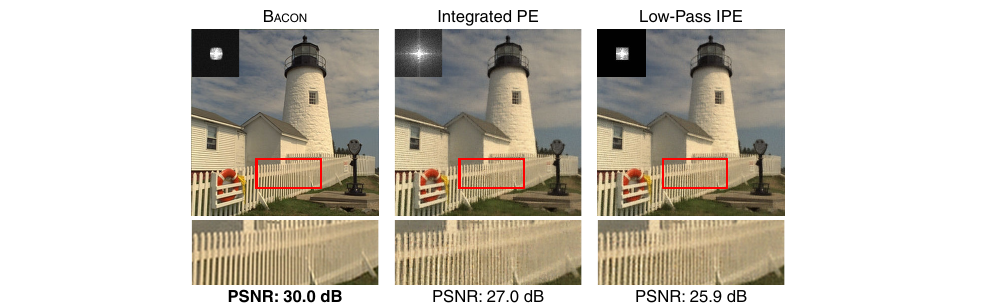}
    \caption{\red{Applying low-pass filter to output of integrated positional encoding (IPE) network. Since IPE networks are not band limited, there are artifacts in the output when upsampling at 4$\times$ resolution (middle column). \bacon{} (left column) is band limited and does not exhibit these artifacts. Spurious high frequency oscillations in the IPE network are aliased onto low-frequency components after sampling the network and cannot be removed by a low-pass filter (right column).}}
    \label{fig:supp_antialias}
\end{figure*}

\begin{figure*}[t]
    \centering
    \includegraphics[]{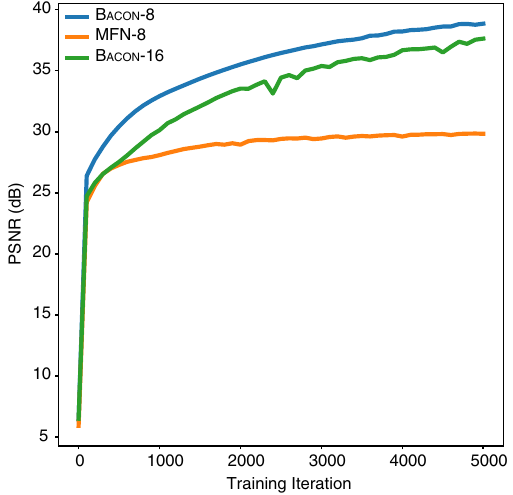}
    \caption{\red{Comparison of deep versions of the original MFN and \bacon{} for image fitting. Both the 8- and 16-layer \bacon{} models fit the lighthouse image to well over 30 dB PSNR while the 8-layer MFN does not reach 30 dB PSNR. We were unable to train a 16-layer MFN with the original initialization scheme due to numerical instabilities during optimization.}}
    \label{fig:supp_deep_mfn}
\end{figure*}

\begin{figure*}[t]
    \centering
    \includegraphics[]{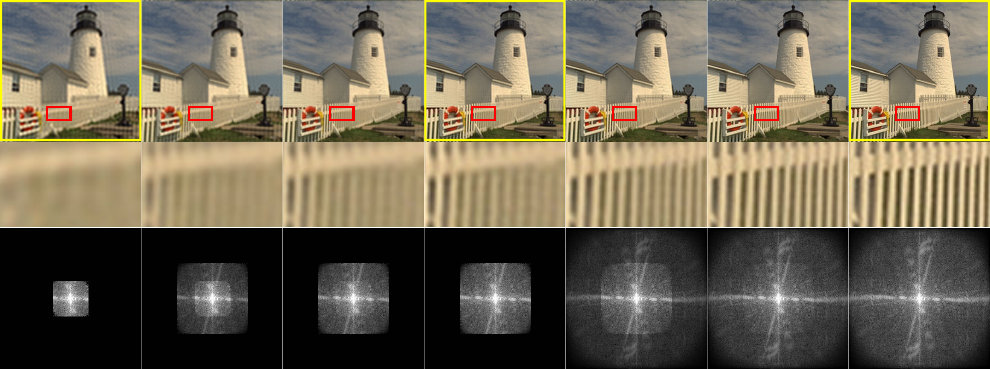}
    \caption{\red{Illustration of interpolation between output scales, similar to the trilinear interpolation used to render from mipmaps. The yellow bordered images in the top row are the outputs of \bacon{} at 1/4, 1/2 and 1$\times$ resolution, and the other images are linear interpolations. Zoomed insets are shown in the middle row, and the bottom row shows the Fourier spectrum of each image.}}
    \label{fig:supp_mipmap}
\end{figure*}

\clearpage

\subsection{Neural Radiance Fields}

We provide additional implementation details and results on reconstructing neural radiance fields using NeRF~\cite{mildenhall2020nerf}, Mip-NeRF~\cite{barron2021mip}, and \bacon{}.

\subsubsection{Additional Implementation Details}
For training neural radiance fields, we use \bacon{} with 8 hidden layers and 256 hidden features.
We set the bandwidth of each layer using random uniform initialization with $\boldsymbol{\omega}_i~\sim \mathcal{U}(-B_i, B_i)$.
For a maximum bandwidth $B$, we set $B_0=B_1=B_2 = B/24$, $B_3=B_4=B/16$, $B_5=B_6=1/8$, and $B_7=B_8=1/4$ such that $\sum_{i=0}^8 B_i = B$.

For training \bacon{} we also adopt the regularization strategy of Hedman et al.~\cite{hedman2021snerg} to penalize non-zero off-surface opacity values, $\sigma$. 
We include this term to mitigate non-zero opacity at unsupervised locations can produce hazy spots in the rendered images using \bacon{}.
The regularization penalty is given as
\begin{equation}
    \mathcal{L}_\text{reg} = \lambda_\sigma \sum\limits_{i,j,k} \log\left(1 + 2\sigma_k(\mathbf{r}_i, \mathbf{t}^f_j) \right),
\end{equation}
where $\lambda_\sigma$ is a weight that we decay from logarithmically from 1e-3 to 1e-6 during training.

\red{To evaluate the effect of the regularization, we train \bacon{}, Mip-NeRF, and NeRF using the same regularizer and report the PSNR in Table~\ref{tab:reg}.
All methods are trained on the \textit{lego} scene for 300K iterations with and without regularization.
NeRF and Mip-NeRF have few opacity artifacts, and so do not benefit from regularization.
\bacon{} shows a significant benefit from regularization.}

Additionally, we find that we can obtain a roughly 30\% speedup for training and inference without noticeable drop in performance by re-using outputs $g_i$ throughout the network.
For example, for layers where $B_i = B_{i+1}$, we set $g_i = g_{i+1}$, allowing us to reuse computation and reducing the number of input layers that need to be computed by roughly half.

\begin{table}[h]
    \centering
    \begin{tabular}{cccc}
        \toprule
        scale & NeRF (no reg/reg) & Mip-NeRF (no reg/reg) & \textsc{Bacon} (no reg/reg)\\ \midrule
         1$\times$  & \textbf{27.695}/27.679 & \textbf{32.655}/32.436 & 24.377/\textbf{29.658} \\ 
         1/2        & \textbf{30.577}/30.565 & \textbf{33.952}/33.863 & 24.611/\textbf{29.501} \\ 
         1/4        & 31.305/\textbf{31.311} & \textbf{34.058}/34.008 & 25.105/\textbf{29.468} \\ 
         1/8        & \textbf{27.067}/27.059 & \textbf{33.805}/33.762 & 25.854/\textbf{28.958} \\ 
         \bottomrule
    \end{tabular}
    \caption{\red{Evaluation of the effect of opacity regularization for NeRF, Mip-NeRF, and \bacon{}}.}
    \label{tab:reg}
\end{table}

\subsubsection{Supplemental Results}
\red{We provide inference times of NeRF, Mip-NeRF and \bacon{} in Table~\ref{tab:timings} evaluated on an NVIDIA RTX A6000 GPU.
 While our implementation is generally slower than the NeRF and Mip-NeRF implementations, we attribute some of this difference to the underlying frameworks; we use PyTorch~\cite{paszke2019pytorch}, while NeRF and Mip-NeRF are implemented in JAX~\cite{jax2018github}.
Additionally, the \bacon{} architecture has somewhat greater computational complexity than the comparable NeRF and Mip-NeRF architectures due to the additional sine input layers and Hadamard products.
Still, for low-resolution outputs \bacon{} has a computational advantage because only the first few layers need to be evaluated.}

In Tables~\ref{table:nerf_psnr} and~\ref{table:nerf_ssim} we provide the per-scene average PSNR and SSIM for each method.
We observe the same trends as in the main paper, with Mip-NeRF achieving the best performance, while \bacon{} outperforms NeRF at the lowest and highest resolution outputs and uses a fraction of the parameters to render the low-resolution outputs compared to either baseline.
\red{An additional comparison is shown in Table~\ref{tab:small_model} for small versions of the NeRF and Mip-NeRF models trained on the \textit{lego} scene.
The number of layers is reduced so that the parameter count is roughly equivalent to the lowest resolution output of \bacon{}.
The output PSNR for each method degrades by roughly 1--2 dB at each scale compared to the full-resolution models}.

Supplemental qualitative results are shown in Fig.~\ref{fig:supp_nerf}, where we show output images for each scene at each scale.
Finally, we show additional results for learning neural radiance fields in a semi-supervised case in Fig.~\ref{fig:supp_nerf_semisupervise}. 
Here, outputs of \bacon{} at each scale are supervised on full resolution images, and \bacon{} automatically learns the multiscale decomposition of the neural radiance field used for rendering.

\begin{table}[h]
    \centering
    \begin{tabular}{lcccc}
    \toprule
    & \multicolumn{4}{c}{Inference Times (s)}\\
                    & 1$\times$ & 1/2 & 1/4 & 1/8 \\\midrule
                    NeRF      &  4.4 & 1.1 & 0.28 & 0.073\\
                    Mip-NeRF  &  4.5 & 1.1 & 0.28 & 0.073 \\
                    \bacon{}  &  10.2 & 2.1 & 0.39 & 0.065 \\
    \bottomrule
    \end{tabular}
    \caption{\red{Inference times for NeRF, Mip-NeRF, and \bacon{}.}}
    \label{tab:timings}
\end{table}

% for 256 samples  
\begin{table}[h]
    \centering
    \begin{tabular}{lc|cccccccc|c}
    \toprule
    & & \multicolumn{8}{c}{PSNR} \\
    & \# Params. & chair  & drums  & ficus  & hotdog  & lego & materials & mic & ship & Avg.\\\midrule
        NeRF 1/8         & \downbar               &  28.767 &  24.025 &  25.188 &   29.685 & 26.539 &      24.758 & 26.720 & 26.028 & 26.464 \\
    NeRF 1/4             & \upbar                 &  33.064 &  25.492 &  26.161 &   33.478 & 30.782 &      26.618 & 30.615 & 28.163 & 29.297 \\
    NeRF 1/2             & 511K                   &  32.882 &  24.503 &  25.387 &   33.711 & 31.037 &      25.850 & 30.517 & 27.640 & 28.941 \\
    NeRF 1$\times$       & \downbar               &  29.565 &  22.741 &  24.280 &   31.408 & 28.228 &      24.319 & 27.827 & 25.508 & 26.734 \\
    NeRF Avg.            & \upbar                 &  31.070 &  24.190 &  25.254 &   32.071 & 29.147 &      25.386 & 28.920 & 26.834 & 27.859 \\\midrule
    Mip-NeRF 1/8         & \downbar               &  37.174 &  28.200 &  28.177 &   37.332 & 33.924 &      30.414 & 35.803 & 31.631 & 32.832 \\
    Mip-NeRF 1/4         & \upbar                 &  36.700 &  26.979 &  26.951 &   37.131 & 34.266 &      29.233 & 34.977 & 30.503 & 32.093 \\
    Mip-NeRF 1/2         & 511K                   &  35.724 &  25.560 &  26.685 &   36.622 & 34.295 &      27.972 & 34.219 & 29.379 & 31.307 \\
    Mip-NeRF 1$\times$   & \downbar               &  33.374 &  24.005 &  26.428 &   34.984 & 33.136 &      26.764 & 32.494 & 27.808 & 29.874 \\
    Mip-NeRF Avg.        & \upbar                 &  35.743 &  26.186 &  27.060 &   36.517 & 33.905 &      28.596 & 34.373 & 29.830 & \textbf{31.526} \\\midrule
    \bacon{} 1/8         & 133K                   &  31.764 &  25.967 &  27.184 &   31.670 & 29.161 &      25.899 & 28.609 & 27.549 & 28.475 \\
    \bacon{} 1/4         & 266K                   &  32.523 &  26.094 &  25.562 &   32.175 & 29.768 &      25.268 & 29.524 & 27.244 & 28.520 \\
    \bacon{} 1/2         & 398K                   &  31.958 &  25.074 &  24.319 &   32.342 & 29.890 &      24.948 & 29.444 & 26.552 & 28.066 \\
    \bacon{} 1$\times$   & 531K                   &  30.729 &  24.175 &  23.753 &   31.942 & 30.418 &      24.300 & 28.454 & 25.668 & 27.430 \\
    \bacon{} Avg.        & 329K                   &  31.744 &  25.327 &  25.204 &   32.032 & 29.809 &      25.104 & 29.008 & 26.753 & \underline{28.123} \\
    \bottomrule
    \end{tabular}
    \caption{PSNR for each method averaged over each scene of the multiscale Blender dataset.}
    \label{table:nerf_psnr}
\end{table}

\clearpage
\begin{table}[h]
    \centering
    \begin{tabular}{lc|cccccccc|c}
    \toprule
    &  & \multicolumn{8}{c}{SSIM} \\
    & \# Params. & chair & drums & ficus & hotdog & lego & materials & mic & ship & Avg.\\\midrule
        NeRF 1/8       & \downbar               & 0.941 &   0.902 &   0.918 &    0.957 &  0.930 &       0.944 & 0.963 &  0.876 & 0.929 \\
    NeRF 1/4           & \upbar                 & 0.976 &   0.926 &   0.947 &    0.973 &  0.968 &       0.943 & 0.980 &  0.888 & 0.950 \\
    NeRF 1/2           & 511K                   & 0.972 &   0.909 &   0.942 &    0.968 &  0.961 &       0.922 & 0.971 &  0.865 & 0.939 \\
    NeRF 1$\times$     & \downbar               & 0.935 &   0.879 &   0.925 &    0.951 &  0.922 &       0.895 & 0.947 &  0.820 & 0.909 \\
    NeRF Avg.          & \upbar                 & 0.956 &   0.904 &   0.933 &    0.962 &  0.945 &       0.926 & 0.965 &  0.862 & 0.932 \\\midrule
    Mip-NeRF 1/8       & \downbar               & 0.990 &   0.952 &   0.951 &    0.986 &  0.984 &       0.978 & 0.994 &  0.928 & 0.970 \\
    Mip-NeRF 1/4       & \upbar                 & 0.989 &   0.942 &   0.954 &    0.983 &  0.984 &       0.964 & 0.990 &  0.909 & 0.964 \\
    Mip-NeRF 1/2       & 511K                   & 0.986 &   0.929 &   0.960 &    0.980 &  0.982 &       0.949 & 0.985 &  0.890 & 0.957 \\
    Mip-NeRF 1$\times$ & \downbar               & 0.975 &   0.915 &   0.957 &    0.973 &  0.972 &       0.932 & 0.980 &  0.861 & 0.946 \\
    Mip-NeRF Avg.      & \upbar                 & 0.985 &   0.935 &   0.955 &    0.981 &  0.980 &       0.955 & 0.987 &  0.897 & \textbf{0.959} \\\midrule
    \bacon{} 1/8       & 133K                   & 0.962 &   0.919 &   0.933 &    0.967 &  0.954 &       0.945 & 0.970 &  0.882 & 0.942 \\
    \bacon{} 1/4       & 266K                   & 0.972 &   0.931 &   0.930 &    0.966 &  0.948 &       0.922 & 0.975 &  0.877 & 0.940 \\
    \bacon{} 1/2       & 398K                   & 0.968 &   0.923 &   0.927 &    0.965 &  0.949 &       0.913 & 0.968 &  0.854 & 0.934 \\
    \bacon{} 1$\times$ & 531K                   & 0.957 &   0.917 &   0.926 &    0.959 &  0.951 &       0.901 & 0.958 &  0.827 & 0.924 \\
    \bacon{} Avg.      & 329K                   & 0.965 &   0.923 &   0.929 &    0.964 &  0.951 &       0.921 & 0.968 &  0.860 & \underline{0.935} \\
    \bottomrule
    \end{tabular}
    \caption{SSIM for each method averaged over each scene of the multiscale Blender dataset.}
    \label{table:nerf_ssim}
\end{table}

\begin{table}[h]
    \centering
    \begin{tabular}{lc|cccc|cccc}
    \toprule
    & & \multicolumn{4}{c|}{PSNR} & \multicolumn{4}{c}{SSIM}\\
    & \# Params. & 1$\times$ & 1/2 & 1/4 & 1/8 & 1$\times$ & 1/2 & 1/4 & 1/8 \\\midrule
        NeRF     & 157K   &  27.144 & 30.050 & 31.554 & 27.309 & 0.903 & 0.949 & 0.971 & 0.940 \\
        Mip-NeRF & 157K   &  30.136 & 32.067 & 32.901 & 32.798 & 0.939 & 0.965 & 0.977 & 0.980 \\
        \bacon{} & 133K   &  N/A & N/A & N/A & 29.161 & N/A & N/A & N/A & 0.954\\ \bottomrule
    \end{tabular}
    \caption{\red{Comparison between small models trained on the \textit{lego} dataset with roughly equal numbers of parameters as the lowest resolution output of \bacon{}. Reducing the parameter count of NeRF and Mip-NeRF results in a roughly 1--2 dB loss at each scale output.}}
    \label{tab:small_model}
\end{table}

\begin{figure*}[t]
    \centering
    \includegraphics{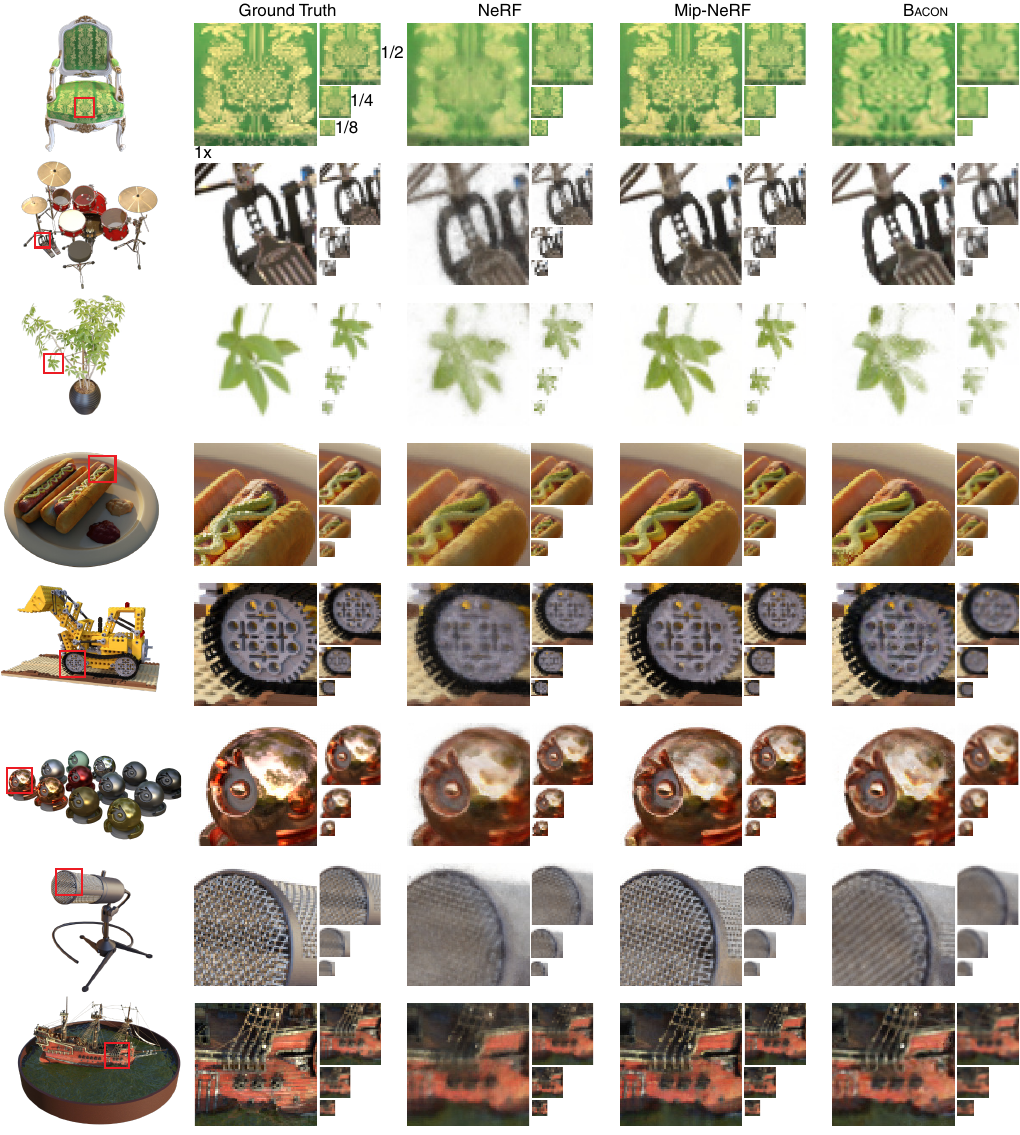}
    \caption{Neural radiance field results. Outputs of NeRF~\cite{mildenhall2020nerf}, Mip-NeRF~\cite{barron2021mip}, and \bacon{} are shown, where all outputs are supervised using each scale of the multiscale Blender dataset.}
    \label{fig:supp_nerf}
\end{figure*}

\begin{figure*}[t]
    \centering
    \includegraphics[width=0.9\textwidth]{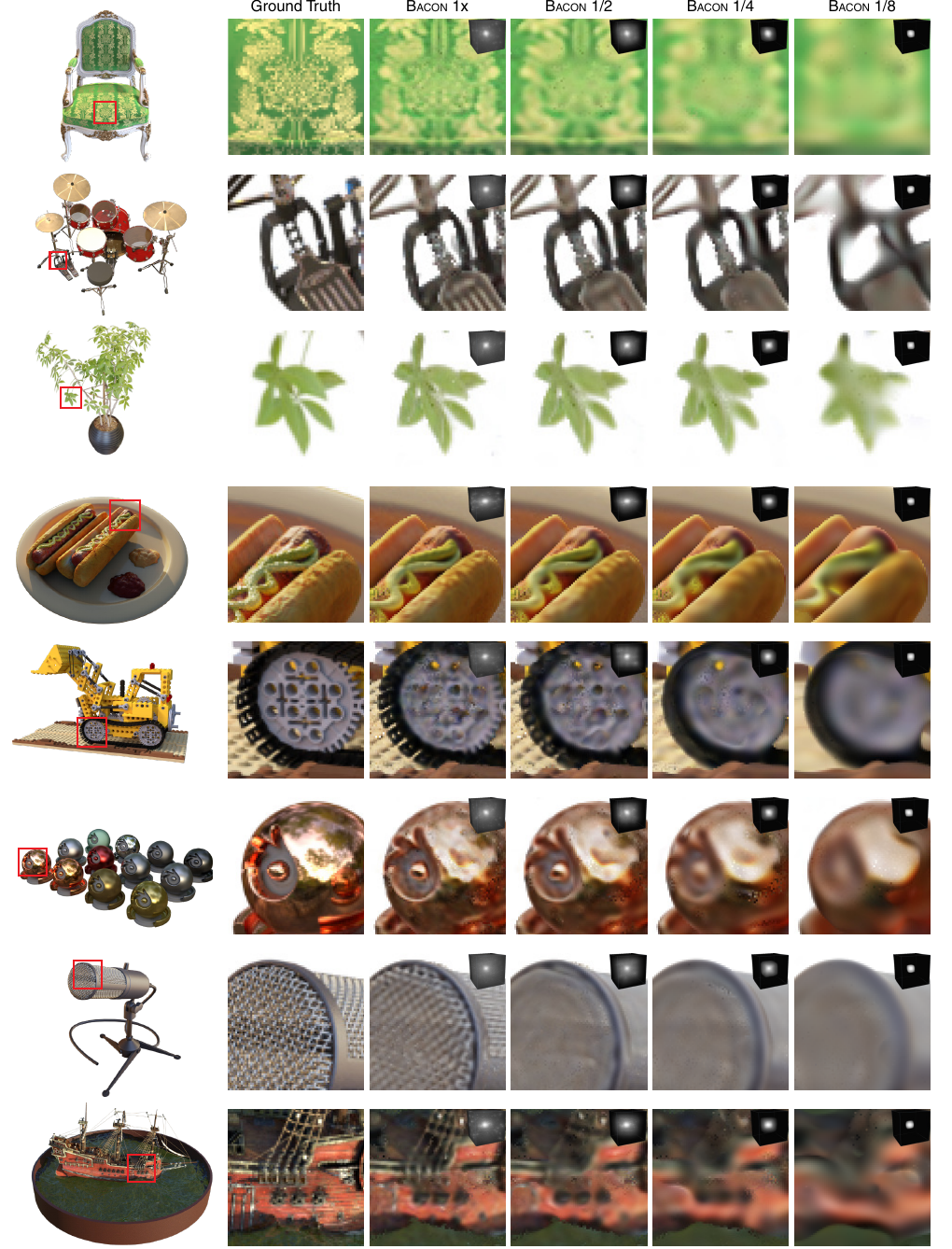}
    \caption{Results of training \bacon{} with all outputs supervised at high resolution. \bacon{} learns a multiscale decomposition for each scene. Fourier spectra of the learned opacity volume are shown as insets.}
    \label{fig:supp_nerf_semisupervise}
\end{figure*}

\clearpage

\subsection{3D Shape Representation}

\subsubsection{Additional Implementation Details}
We evaluate \bacon{} on 3D shape representation with SDFs and compare its performance to three other methods: Fourier Features~\cite{tancik2020fourier}, \siren{}~\cite{sitzmann2020siren}, and Neural Geometric Level of Detail (\nglod{})~\cite{takikawa2021nglod}.
All networks are trained to directly fit a signed distance function estimated from a ground-truth mesh.
For \bacon{}, Fourier Features, and \siren{} we use 8 hidden layers with 256 hidden features.
We scale the shape models so that they fit within a volume whose dimensions extend from -0.5 to 0.5.
The maximum bandwidth of \bacon{} is set to 256 cycles per unit interval for the \textit{Lucy} and \textit{Thai Statue} scenes, and to 192 cycle per unit interval for the \textit{Armadillo}, \textit{Dragon}, and \textit{Sphere} models.
We scale the bandwidth of each layer the same as with the neural radiance field models as described in the main text.
We determine the maximum bandwidth empirically such that the network achieves a good fit to high frequency details in the model with noticeable smoothing in the lower levels of detail.

After training, the models are extracted at $512^3$ resolution using marching cubes, and we evaluate performance using Chamfer distance and intersection over union (IOU).
Chamfer distance is calculated by sampling 300,000 points on the surface of the ground truth and predicted models and then finding the distance to the closest point on the other surface.
That is, for the two point clouds $P_1$, and $P_2$ we compute 
\begin{equation}
D_\text{Chamfer}(P_1, P_2) = \frac{1}{|P_1|} \sum\limits_{\mathbf{x}\in P_1} \min_{\mathbf{y}\in P_2}\lVert \mathbf{x} - \mathbf{y} \rVert_2^2 + \frac{1}{|P_2|} \sum\limits_{\mathbf{x}\in P_2} \min_{\mathbf{y}\in P_1}\lVert \mathbf{x} - \mathbf{y} \rVert_2^2.
\end{equation}
For the IOU score, we compute intersection and union of the occupancy values for the ground truth and predicted meshes on a $128^3$ grid of points centered on the object.

\subsubsection{Supplemental Results}
We include shape fitting results for four scenes from the Stanford 3D Scanning Repository (\textit{Armadillo}, \textit{Dragon}, \textit{Lucy}, \textit{Thai Statue}) and a simple sphere in Figures~\ref{fig:supp_dragon},~\ref{fig:supp_armadillo},~\ref{fig:supp_lucy}, \ref{fig:supp_thai}, and \ref{fig:supp_sphere}.
All methods perform similarly at the highest level of detail (see Table~\ref{tab:shapes_quant}).
For lower levels of detail, \bacon{} (1/8, 1/4, 1/2) represents a smooth multiscale decomposition of the shape, while the representations for \nglod{} (\nglod{}-1,2,3) show angular artifacts due to their high-frequency spectra (see figure insets).
Note that results for \nglod{} are shown for training the representation on a maximum of 4 levels of detail (i.e., the number of trained levels of their feature octree) and then rendering out the resulting trained levels of detail 1--4.

Table~\ref{tab:shapes_quant} includes quantitative evaluation of each method for the \textit{Armadillo}, \textit{Dragon}, \textit{Lucy}, \textit{Thai Statue} scenes, and a simple sphere baseline (with radius 0.25).
The highest detail outputs of all methods perform comparably, including \bacon{}, which achieves similar performance despite simultaneously representing multiple levels of detail.
\nglod{} generally improves at higher levels of detail at the cost of significantly more model parameters.
Here, \nglod{} 1--4 represent outputs from the model trained at maximum level of detail 4, and we also train separate models with maximum levels of detail 5--6 (\nglod{}-5 and \nglod{}-6).

Additionally, Table~\ref{tab:shapes_quant_low_res} includes evaluation of shape fitting at lower levels of detail for \bacon{} and \nglod{} (trained with a maximum level of detail 4).
Note that here \nglod{} has fewer parameters than \bacon{} for lower levels of detail; this is not true for the full resolution models, as number of parameters scales superlinearly for \nglod{} and linearly for \bacon{}.

\begin{table}[h]
    \centering
    \resizebox{\columnwidth}{!}{
    \begin{tabular}{lccccccccccc}
    \toprule

                    & \multirow{2}{*}{Parameters} & \multicolumn{2}{c}{Sphere} & \multicolumn{2}{c}{Dragon} & \multicolumn{2}{c}{Armadillo} & \multicolumn{2}{c}{Lucy} & \multicolumn{2}{c}{Thai Statue} \\
                    & & Chamfer$\downarrow$ & IOU$\uparrow$ & Chamfer$\downarrow$ & IOU$\uparrow$  & Chamfer$\downarrow$ & IOU$\uparrow$ & Chamfer$\downarrow$ & IOU$\uparrow$ & Chamfer$\downarrow$ & IOU$\uparrow$ \\\midrule

        Fourier Features    & 527K  & \underline{8.364e-7} & \textbf{1.000e+0}    &   \textbf{1.861e-6}     & 9.828e-1             & \underline{3.230e-6}  & \underline{9.897e-1}  & \textbf{2.956e-6}     & \textbf{9.654e-1}     & \textbf{1.946e-6}     & \underline{9.823e-1}\\
        \textsc{Siren}      & 528K  & 8.407e-7          & \textbf{1.000e+0}        &   2.762e-6             & 9.621e-1             & 3.895e-6              & 9.858e-1              & 3.706e-6              & 9.625e-1              & 2.695e-6              & 9.651e-1 \\
        \nglod{}-4             & 1.35M & 9.722e-7          & 9.990e-1                 &   2.272e-6             & 9.722e-1             & 3.410e-6              & 9.891e-1              & 5.479e-6              & 9.421e-1              & 2.320e-6              & 9.689e-1 \\
        \nglod{}-5             & 10.1M & 9.443e-7          & \underline{9.993e-1}     &   2.211e-6             & \textbf{9.841e-1}    & 3.804e-6              & 9.835e-1              & 3.206e-6              & 9.621e-1              & 2.047e-6              & 9.767e-1 \\
        \nglod{}-6             & 78.8M & 1.064e-6          & 9.966e-1                 &   1.918e-6             & \underline{9.840e-1} & \textbf{3.212e-6}     & \textbf{9.911e-1}     & \underline{3.013e-6}  & 9.634e-1              & \underline{1.939e-6}  & \textbf{9.824e-1} \\
        \bacon{} 1$\times$  & 531K  & \textbf{8.353e-7} & \textbf{1.000e+0}        &   \underline{1.875e-6} & 9.831e-1             & 3.233e-6              & 9.893e-1              & 3.075e-6              & \underline{9.650e-1}  & 1.972e-6              & 9.791e-1 \\
    \bottomrule
\end{tabular}}
\caption{Quantitative evaluation of 3D shape fitting for high-detail outputs. All methods show comparable performance. Neural Geometric Level of Detail (\nglod{})~\cite{takikawa2021nglod} is shown for levels of detail 4, 5, and 6 and performance generally increases (as does the model parameter count) with increasing levels of detail. \bacon{} gives comparable performance at the highest resolution scale to other methods despite simultaneously representing multiple levels of detail.}
    \label{tab:shapes_quant}
\end{table}

\begin{table}[h]
    \centering
    \resizebox{\columnwidth}{!}{
    \begin{tabular}{lccccccccccc}
    \toprule
         & \multirow{2}{*}{Parameters} & \multicolumn{2}{c}{Sphere} & \multicolumn{2}{c}{Dragon} & \multicolumn{2}{c}{Armadillo} & \multicolumn{2}{c}{Lucy} & \multicolumn{2}{c}{Thai Statue} \\
         & & Chamfer$\downarrow$ & IOU$\uparrow$ & Chamfer$\downarrow$ & IOU$\uparrow$  & Chamfer$\downarrow$ & IOU$\uparrow$ & Chamfer$\downarrow$ & IOU$\uparrow$ & Chamfer$\downarrow$ & IOU$\uparrow$ \\\midrule
	    \nglod{}-1             & 8.74K & 9.391e-7              & 9.987e-1             & 6.624e-6             & 9.381e-1             & 6.435e-6             & 9.699e-1              & 1.965e-5             &   8.936e-1           & 8.139e-6             & 9.392e-1 \\
        \nglod{}-2             & 36.8K & 9.549e-7              & 9.989e-1             & 3.247e-6             & 9.612e-1             & 4.033e-6             & 9.839e-1              & 7.550e-6             &   9.288e-1           & 3.474e-6             & 9.638e-1 \\
        \nglod{}-3             & 199K  & 1.100e-6              & \underline{9.975e-1} & 2.274e-6             & 9.722e-1             & 3.407e-6             & 9.891e-1              & 5.513e-6             &   9.421e-1           & 2.325e-6             & 9.689e-1 \\
        \bacon{} 1/8        & 133K  & 8.349e-7              & \textbf{1.000e+0}    & 3.430e-6             & 9.624e-1             & 3.997e-6             & 9.844e-1              & 5.309e-6             & 9.461e-1             & 3.168e-6             & 9.622e-1 \\
        \bacon{} 1/4        & 266K  & \textbf{8.320e-7}     & \textbf{1.000e+0}    & \underline{2.223e-6} & \underline{9.773e-1} & \underline{3.355e-6} &\underline{ 9.892e-1}  & \underline{3.467e-6} & \underline{9.627e-1} & \underline{2.119e-6} & \underline{9.748e-1} \\
        \bacon{} 1/2        & 398K  & \underline{8.343e-7} & \textbf{1.000e+0}    & \textbf{1.955e-6}    & \textbf{9.815e-1}    & \textbf{3.242e-6}    & \textbf{9.897e-1}     & \textbf{3.174e-6}    & \textbf{9.652e-1}    & \textbf{1.985e-6}    & \textbf{9.799e-1} \\
    \bottomrule
\end{tabular}}
\caption{Quantitative evaluation of 3D shape fitting for \bacon{} and Neural Geometric Level of Detail (\nglod{})~\cite{takikawa2021nglod} for low levels of detail. Here, \nglod{} has fewer parameters than \bacon{}, and \bacon{} generally achieves better performance.}
    \label{tab:shapes_quant_low_res}
\end{table}

\clearpage
\begin{figure*}[ht]
    \centering
    \includegraphics[width=\textwidth]{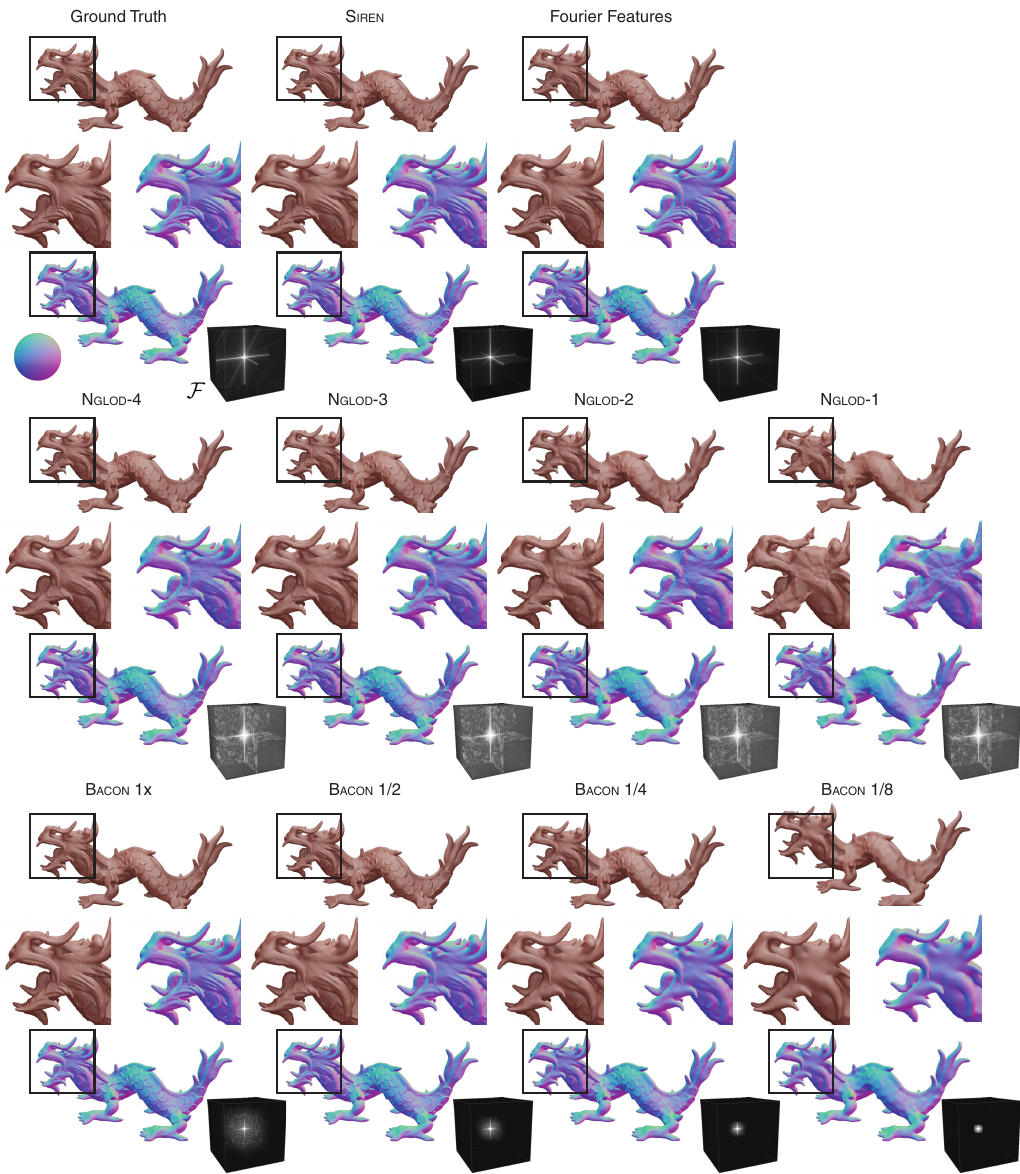}
    \caption{Qualitative results on the \textit{Dragon} scene. Rendered objects and normal maps are shown, and Fourier spectra of the SDF values are included as insets.}
    \label{fig:supp_dragon}
\end{figure*}

\begin{figure*}[ht]
    \centering
    \includegraphics[width=\textwidth]{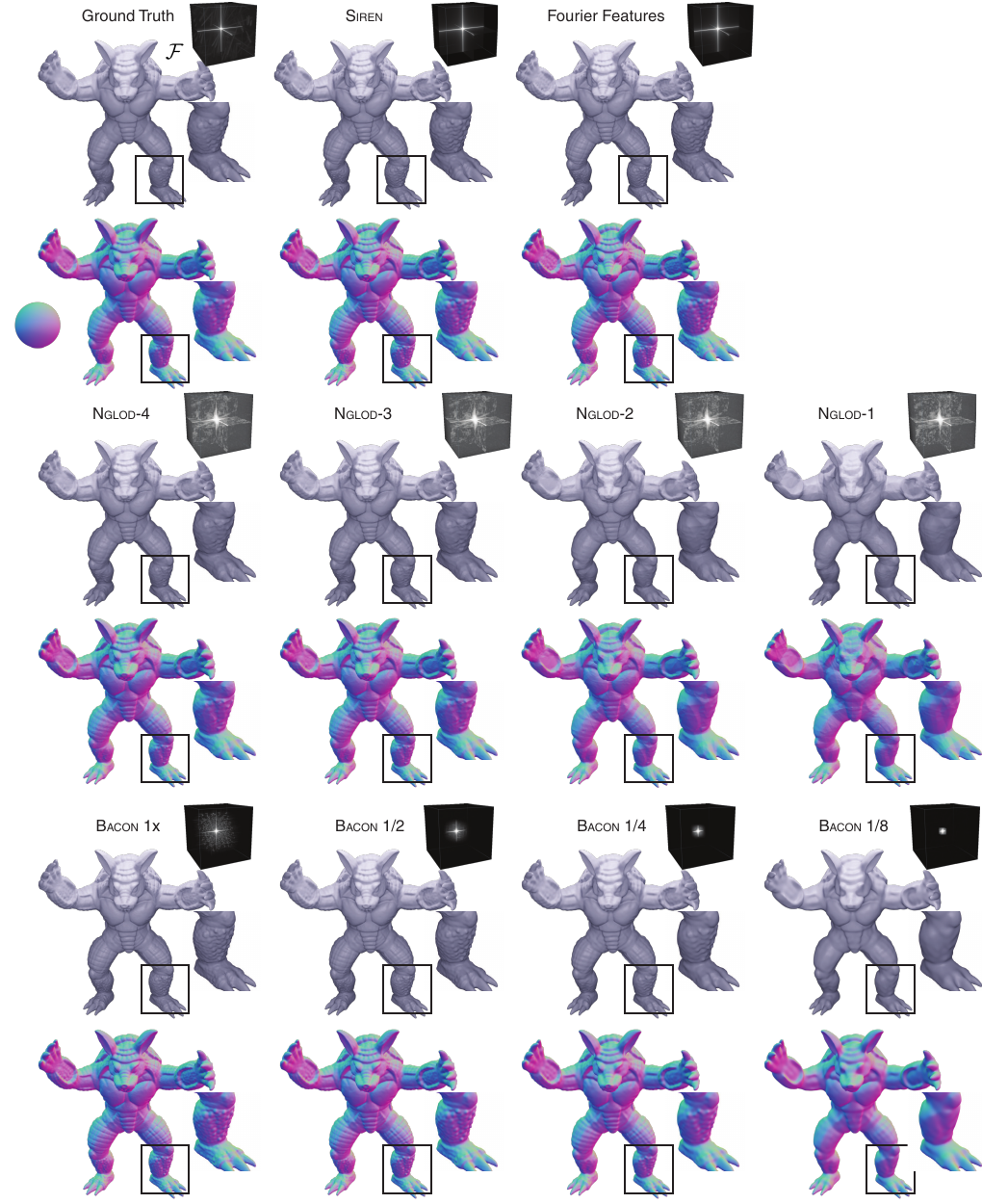}
    \caption{Qualitative results on the \textit{Armadillo} scene. Rendered objects and normal maps are shown, and Fourier spectra of the SDF values are included as insets.}
    \label{fig:supp_armadillo}
\end{figure*}

\begin{figure*}[ht]
    \centering
    \includegraphics[width=0.95\textwidth]{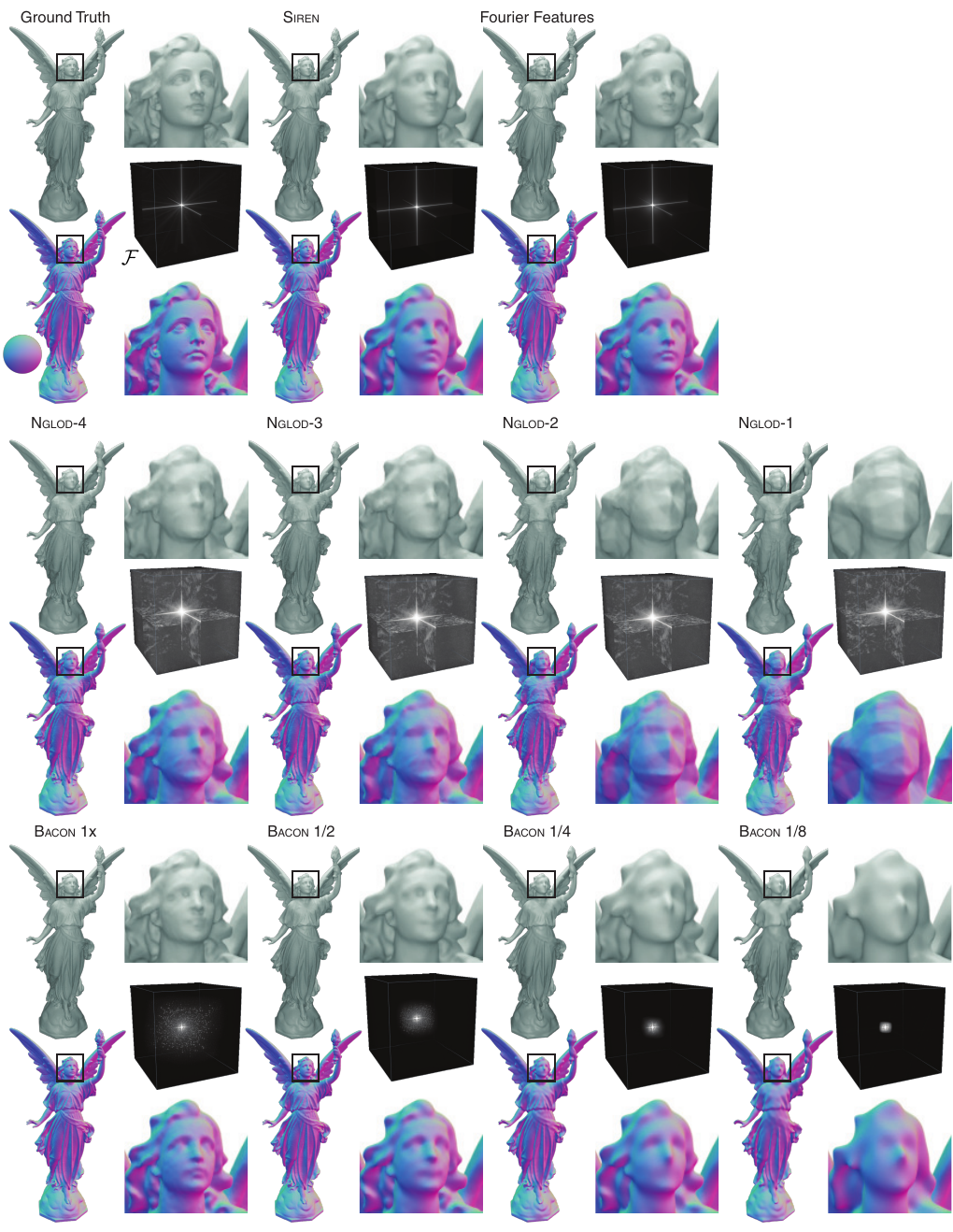}
    \caption{Qualitative results on the \textit{Lucy} scene. Rendered objects and normal maps are shown, and Fourier spectra of the SDF values are included as insets.}
    \label{fig:supp_lucy}
\end{figure*}

\begin{figure*}[ht]
    \centering
    \includegraphics[width=0.95\textwidth]{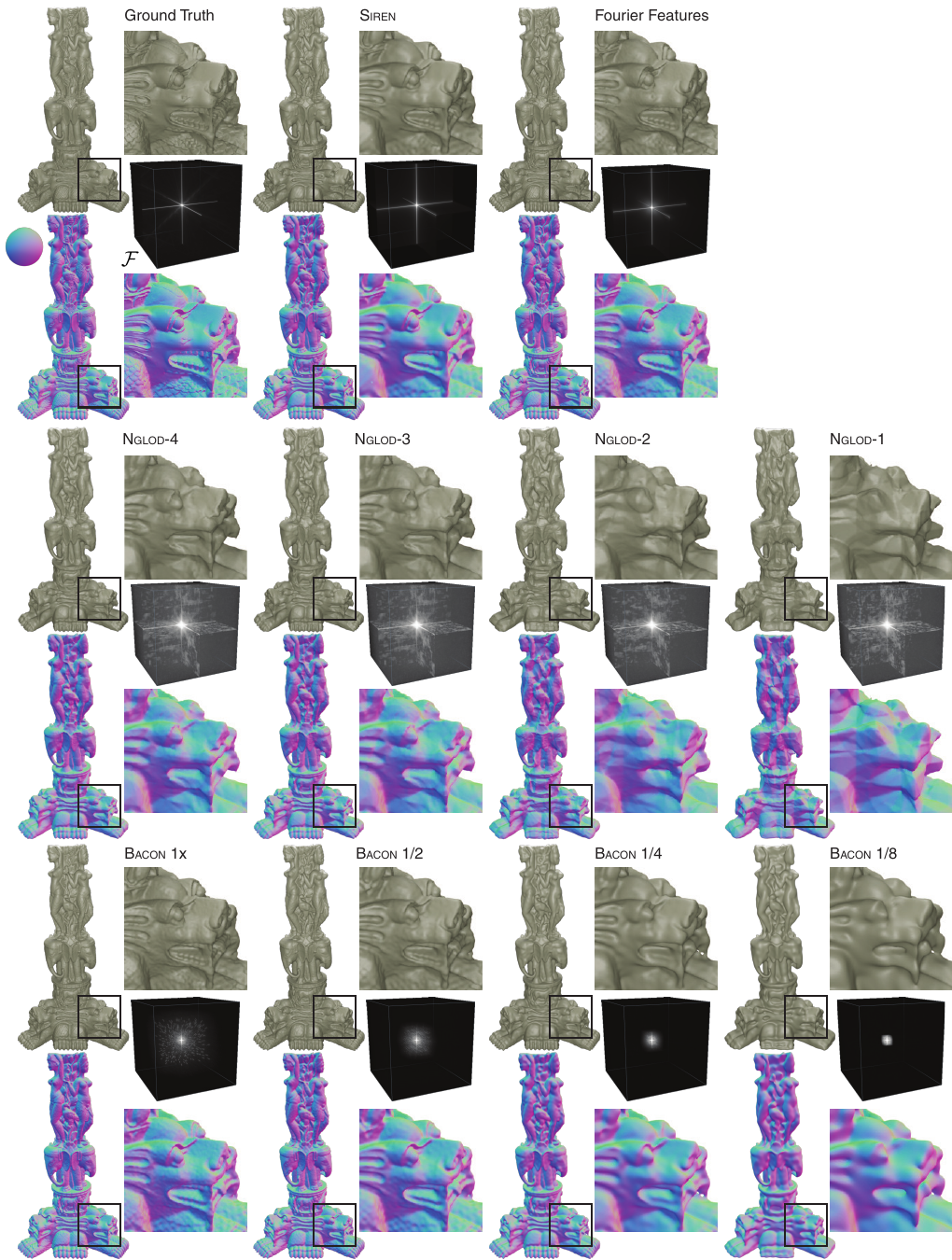}
    \caption{Qualitative results on the \textit{Thai Statue} scene. Rendered objects and normal maps are shown, and Fourier spectra of the SDF values are included as insets.}
    \label{fig:supp_thai}
\end{figure*}

\begin{figure*}[ht]
    \centering
    \includegraphics[width=0.95\textwidth]{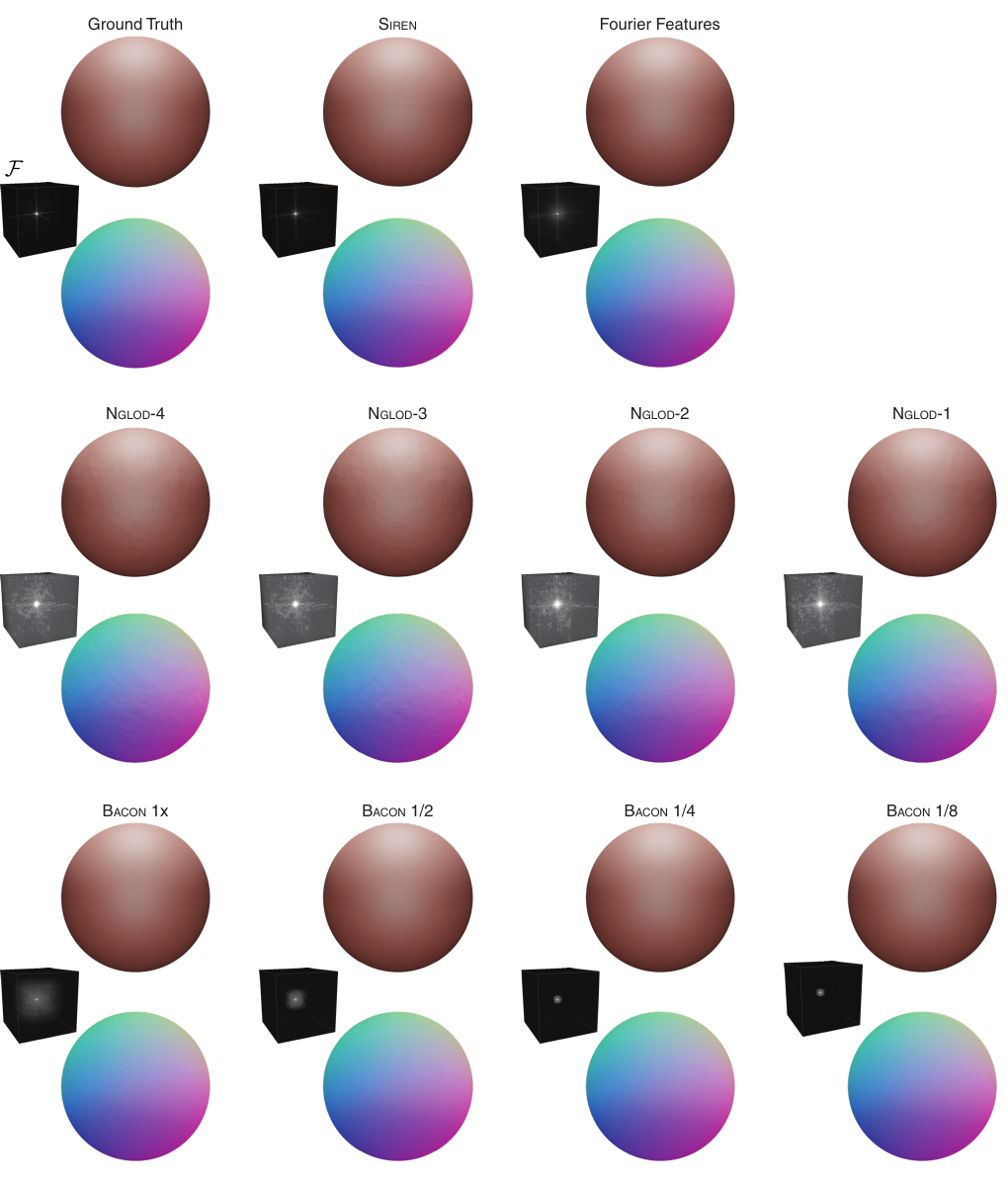}
    \caption{Qualitative results on the \textit{Sphere} scene. Rendered objects and normal maps are shown, and Fourier spectra of the SDF values are included as insets.}
    \label{fig:supp_sphere}
\end{figure*}

\clearpage

\subsection{Accelerated Marching Cubes}
In this section, we explain two strategies for accelerating mesh extraction via the Marching Cubes algorithm with signed distance function (SDF) representation networks.

\subsubsection{Adaptive-Frequency SDF Evaluation}
We observe that the band-limited, multi-scale nature of our network allows efficient allocation of computational resources when evaluating SDFs.
The key idea is to adaptively choose whether to use SDFs from low-frequency or high-frequency output layers (Fig.
\ref{fig:adaptive}).
For each cell, we compute the low-frequency output $\text{SDF}_{\text{low}}$ that takes a fraction of time of the full network evaluation.
Then, for cells that are far away from the zero-level-set (i.e., magnitude of $\text{SDF}_{\text{low}}$ larger than some threshold $\tau$), we adopt early-stopping and do not proceed to the higher network layers, as we do not need highly-accurate SDFs for empty cells.
For cells near the surface (i.e., $|\text{SDF}_{\text{low}}|<\tau$), we need accurate SDFs, and thus we use the full network for high-frequency outputs. 
This adaptive early-stopping strategy meaningfully reduces the computation time for mesh extraction (Table~\ref{tab:MarchinbCubes}) and is unique to \bacon{} that produces multi-scale intermediate outputs, unlike the existing architectures such as SIREN that needs to go through the entire network for all cases.
We set $\tau$ to be 0.7 times the finest voxel length.

\subsubsection{Multi-scale SDF Evaluation}
We introduce another strategy to accelerate Marching Cubes mesh extraction using octree-style, multi-scale SDF grids.
As shown in Fig.~\ref{fig: MS_MC}, we evaluate the shape SDFs in a hierarchical way, from the low to high resolution grids.
We note that the SDFs evaluated at a coarse level can be used to decide whether or not to subdivide a cell.
That is, assuming the modeled SDFs are accurate, when the magnitude of SDF at the center of a voxel is larger than the radius of the circumsphere, the voxel is empty (i.e., containing no zero-level-set), so we do not need to further evaluate the SDFs at higher resolutions.
Similarly, when the SDF magnitude at the center is smaller than the threshold $R$, the voxel contains zero-crossing, so it needs to be evaluated at higher resolution via subdivision.
Empirically, we set $R$ to be 2 times the circumsphere radius, to provide a margin of safety to the SDF modeling errors.
This multi-scale Marching Cubes approach is not unique to \bacon{} and can be applied to other SDF-modeling networks such as SIREN.
As shown in Table~\ref{tab:MarchinbCubes}, the strategy reduces the computation time by a factor of $\approx$40. 

\subsubsection{Combining the Two Strategies}
While the above two strategies individually provide significant speedup for mesh extraction, we can combine them together to further enhance the performance.
That is, we adopt the adaptive-frequency approach for each level in the multi-scale evaluation.
For coarse levels, we adopt the early-stopping strategy to all cells.  For the finest resolution level, which takes account for most of the computations, we similarly adopt early-stopping for voxels that are far away from the zero-crossings using the threshold $\tau$.
As a result, the combination of the two strategies provide another meaningful reduction of computation time against the pure multi-scale scheme, as shown in Table~\ref{tab:MarchinbCubes}.
Note our accelerated Marching Cubes does not decrease the quality of the extracted meshes (see, output shapes in Fig.~\ref{fig:supp_mc}).

\subsubsection{Discussion of Occupancy Networks}
We notice that Occupancy Networks similarly proposed a multi-resolution mesh extraction strategy on the occupancy fields.
The octree-style evaluation for occupancy fields could lead to errors, however, because occupancy fields to not provide the same empty-space guarantees that SDFs provide.
Furthermore, the adaptive-frequency evaluation cannot be used for Occupancy Networks, and thus all query points need to be evaluated by the full network layers.

\begin{figure*}[ht]
\centering
\hspace*{\fill}
\begin{subfigure}[t]{.35\linewidth}
  \centering
  \includegraphics[width=.99\linewidth]{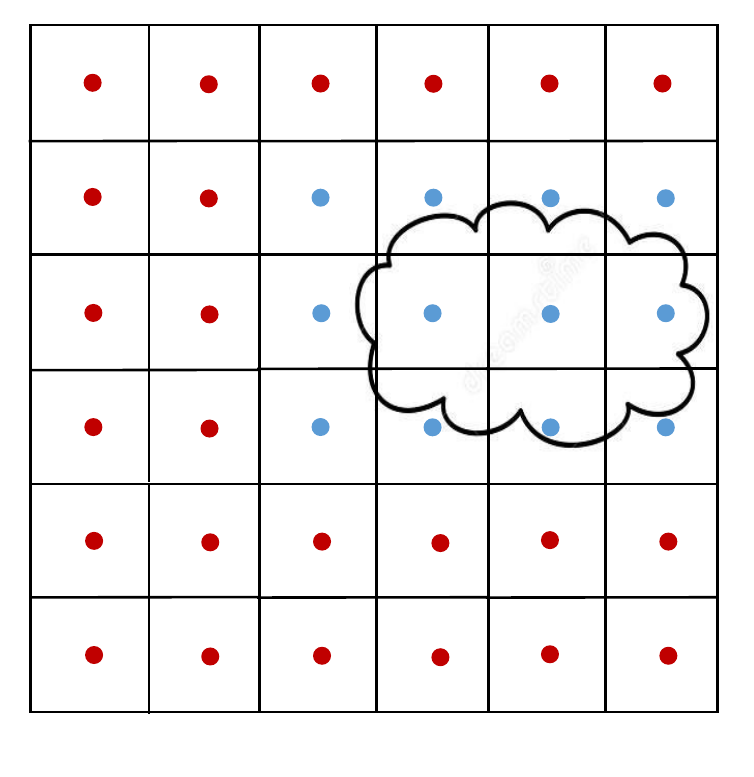}
  \caption{Adaptive-Frequency Outputs}
  \label{fig:sub1}
\end{subfigure}%
\hfill
\begin{subfigure}[t]{.55\linewidth}
  \centering
  \includegraphics[width=.99\linewidth]{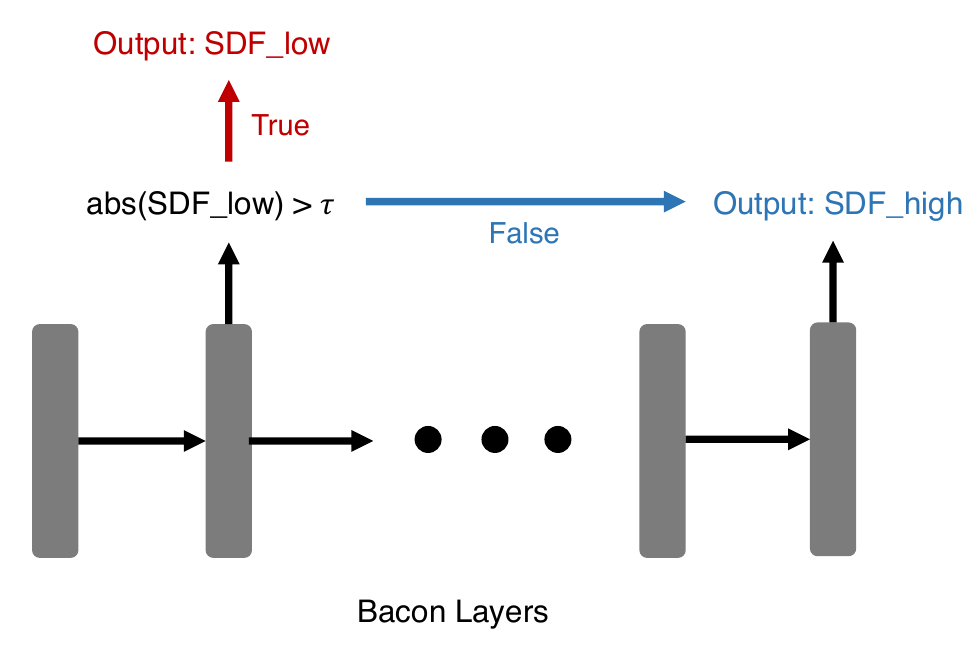}
  \caption{Adaptive-Frequecy Decision Process}
  \label{fig:sub2}
\end{subfigure}
\hspace*{\fill}
\caption{Adaptive-frequency SDF evaluations. When evaluating SDFs on a dense grid (a) for mesh extraction, we leverage the band-limited nature of \bacon{} layers to adaptively allocate the computation resources (b) across the cells. For each cell we first compute the SDF with a low-frequency output layer ($\text{SDF}_{\text{low}}$). For cells with the magnitude of $\text{SDF}_{\text{low}}$ larger than some threshold $\tau$ (i.e., the red cells that are far from the zero-level-set), we do not proceed to the higher layers of \bacon{}, as we do not need high-frequency details in the empty-space. For cells near the surface (i.e., when $|\text{SDF}_{\text{low}}|<\tau$), we compute the full-frequency SDF ($\text{SDF}_{\text{high}}$) using the highest layer output (the blue cells). This adaptive-frequency SDF evaluation saves significant amount of time on computing the SDFs in empty-space via early-stopping, which cannot be adopted by existing network architectures, e.g., SIREN, that need to go through the entire network layers for all evaluations.}
\label{fig:adaptive}
\end{figure*}

\begin{figure*}[ht]
    \centering
    \includegraphics[width=\textwidth]{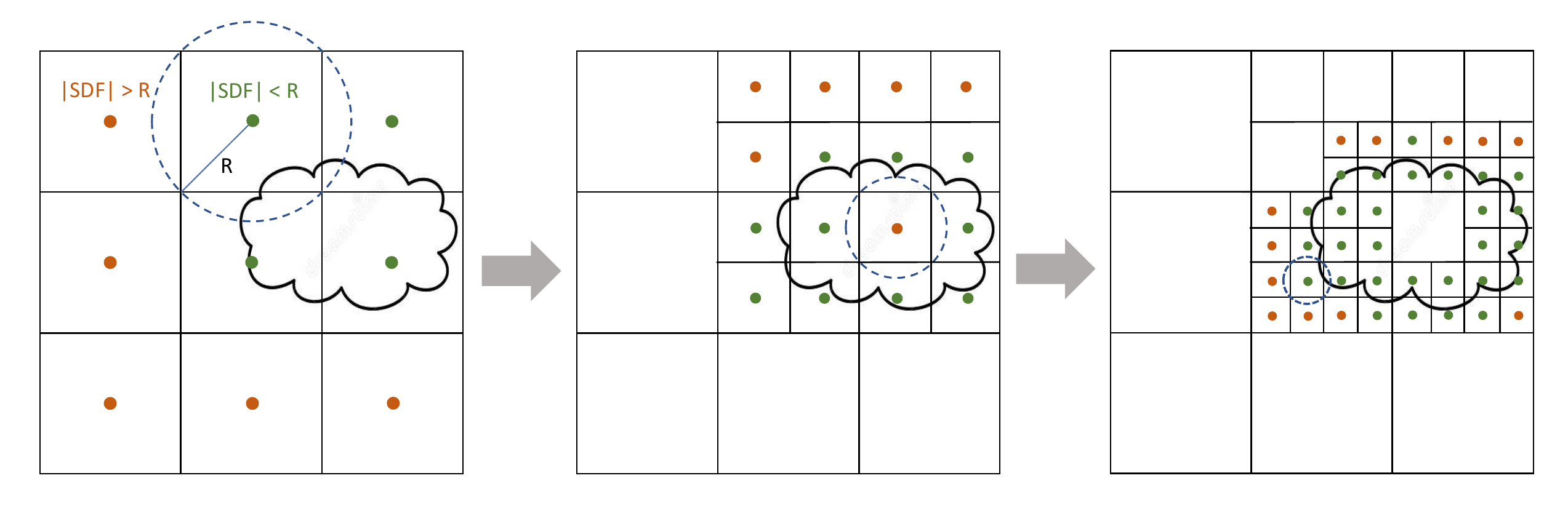}
    \caption{Multi-Scale Marching Cubes. We propose a hierarchical octree-style SDF evaluation scheme for efficiently pruning empty spaces using the nature of SDFs. We start from the coarsest resolution grid (left) and query the SDFs for all cells. Assuming the modeled SDFs are correct, we can identify some of the cells to be empty (i.e, no zero-crossing exists within the cell), when the magnitude of SDFs at the center is larger than the the radius of the circumsphere (dotted circles in the diagram). In practice, to provide a margin of safety for the model errors, we use $2\times R$ for the criteria for checking empty cells. Then, only for the non-empty cells (green), we subdivide them into 8 cells and evaluate higher resolution SDFs in the next level grid, which we repeat multiple times. For all our experiments we used 4 levels of scales. Note that the proposed multi-scale SDF evaluation is not unique to \bacon{} and can be applied to existing networks, e.g., SIREN, that model SDFs.}\label{fig: MS_MC}
\end{figure*}

\begin{table}[h]
    \centering
    \begin{tabular}{lcccccccc}
    \toprule
    & \multicolumn{5}{c}{Seconds} \\
    & Armadillo & Dragon & Lucy & Sphere & Thai\\\midrule
    SIREN Original         & 16.75 &   16.78 &   16.76 &    16.78 &  16.76 \\ 
    SIREN Multi-Scale               & 0.404 &   0.258 &   0.253 &    0.252 &  0.354  \\
    \bacon{} Original                & 17.93 &   17.836 &   17.909 &    17.926 &  17.938 \\ 
    \bacon{} Adaptive               & 5.925 &   5.325 &   5.278 &    5.391 &  5.584  \\ 
    \bacon{} Multi-Scale                & 0.411 &   0.273 &   0.270 &    0.265 &  0.364  \\ 
    \bacon{} Adapt. + Multi. (Proposed)                & \textbf{0.280} &   \textbf{0.207} &   \textbf{0.188} &    \textbf{0.172} &  \textbf{0.267}  \\ 
    \bottomrule
    \end{tabular}
    \caption{Marching Cubes timing analysis. From top to bottom: dense vanilla Marching Cubes using SIREN; multi-scale Marching Cubes using SIREN; dense vanilla Marching Cubes  using \bacon{}; adaptive mesh extraction using \bacon{}; multi-scale Marching Cubes using \bacon{}; the proposed combination of multi-scale and adaptive SDF evaluation using \bacon{}. } \label{tab:MarchinbCubes}
\end{table}

\begin{figure*}[ht]
    \centering
    \includegraphics[width=\textwidth]{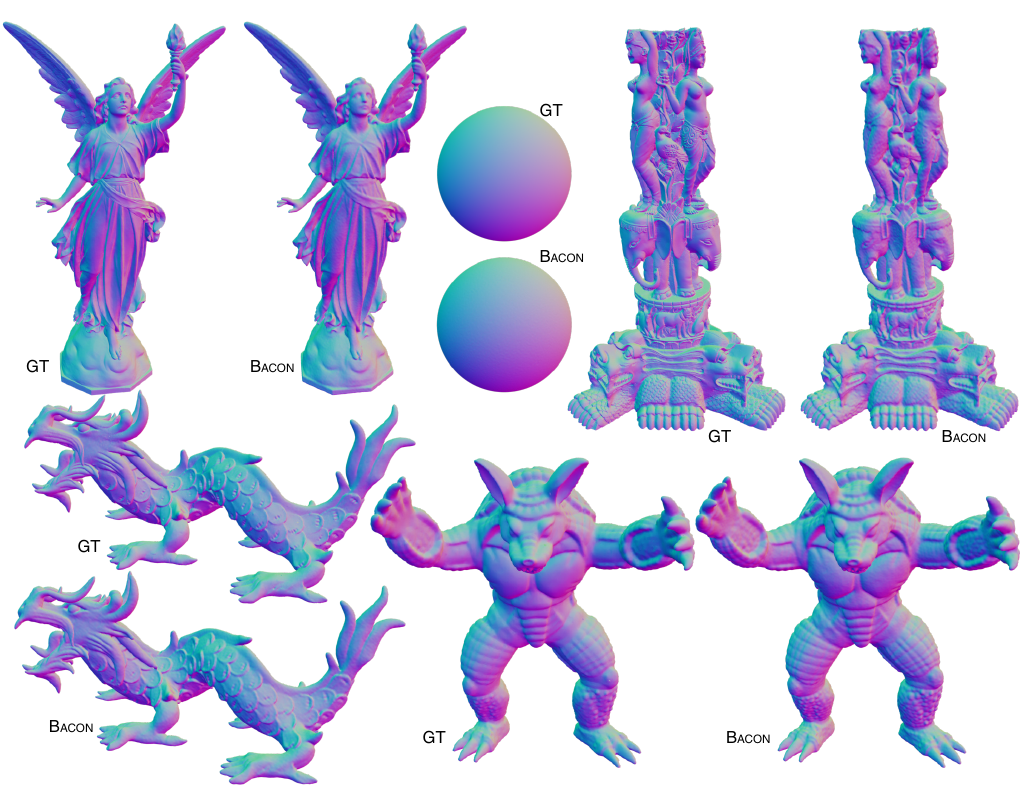}
    \caption{Extracted 3D shapes using the proposed adaptive-frequency multiscale inference procedure.}
    \label{fig:supp_mc}
\end{figure*}

\clearpage

\begin{figure}[ht]
    \centering
    \includegraphics[]{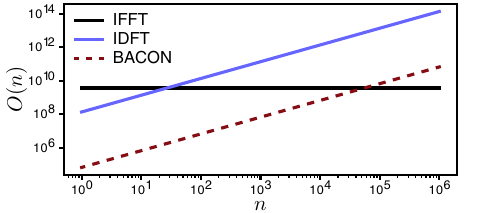}
    \caption{Comparison of asymptotic computational complexity.}
    \label{fig:fft}
\end{figure}

\subsection{Comparison to Explicit Fourier Basis}
Interestingly, we find that a \bacon{} representing a grid of $512^3$ discrete frequencies (used in our shape fitting experiments) is more efficient to evaluate for few samples than using the Inverse Fast Fourier Transform (IFFT) or naive computation of the inverse discrete Fourier transform (IDFT) on an explicit grid of $512^3$ coefficients.
The computational complexity of the IFFT and IDFT are $O(N\log(N))$ and $O(N^2)$, where here, $N=512^3$.
\bacon{} is a compressive representation of the spectrum, and its complexity scales as $O(d_\text{h}^2)$ (due to matrix multiplication).
In Fig.~\ref{fig:fft} we plot a comparison of asymptotic computational complexity for $n$ output samples for each of these methods. 
Since the IFFT always computes $512^3$ outputs, its cost is constant.

{\small
\bibliographystyle{ieee_fullname}
\bibliography{references}
}